%% file: main.tex
\documentclass[11pt]{article}

\usepackage[T1]{fontenc}
\usepackage[utf8]{inputenc}
\usepackage{lmodern}

\usepackage{geometry}

\usepackage{microtype}

\input{math_commands.tex}

\title{Online Decision Making with Generative Action Sets}

\author{
Jianyu Xu \\
Carnegie Mellon University\\
Pittsburgh, PA 15213 \\
\texttt{jianyux@andrew.cmu.edu}
\and
Vidhi Jain \\
Carnegie Mellon University\\
Pittsburgh, PA 15213 \\
\texttt{vidhij2@andrew.cmu.edu}
\and
Bryan Wilder \\
Carnegie Mellon University\\
Pittsburgh, PA 15213 \\
\texttt{bwilder@andrew.cmu.edu}
\and
Aarti Singh \\
Carnegie Mellon University\\
Pittsburgh, PA 15213 \\
\texttt{aarti@andrew.cmu.edu}
}

\date{}

\begin{document}

\maketitle

\begin{abstract}
\input{abstract}
\end{abstract}

\newpage

\section{Introduction}
\label{sec:introduction}
\input{introduction}

\section{Related Works}
\label{sec:related_works}
\input{related_works}

\section{Problem Setup}
\label{sec:problem_setup}
\input{problem_setup}

\section{Algorithm}
\label{sec:algorithm}
\input{algorithm}

\section{Theoretical Analysis}
\label{sec:regret_analysis}
\input{theoretical_analysis}

\section{Empirical Performance}
\label{sec:empirical}
\input{empirical}

\section{Discussions}
\label{sec:discussion}
\input{discussion}

\section{Conclusion}
\label{sec:conclusion}
\input{conclusion}

\section*{Acknowledgments} This work is partially supported by the Engler Family Foundation and the NSF AI Institute for Societal Decision Making (NSF AI-SDM, Award No. 2229881). We appreciate the input from \texttt{Nivi.Inc} for providing the maternal health dataset.

\bibliographystyle{plainnat}
\bibliography{ref}

\appendix
\newpage
\crefalias{section}{appendix}
\crefalias{subsection}{appendix}
\crefalias{theorem}{lemma}
\input{appendix}

\end{document}

%% file: math_commands.tex
\usepackage{amsmath,amsfonts,amsthm,bm,bbm}
\usepackage{mathtools}

\usepackage{graphicx}
\usepackage{caption}
\usepackage{subcaption}

\usepackage{enumerate}
\usepackage{enumitem}
\setlist{leftmargin=10mm}

\usepackage{algorithm,algorithmic}

\usepackage{comment}
\usepackage{url}

\usepackage{natbib}
\usepackage[colorlinks=true,linkcolor=blue,citecolor=blue,urlcolor=blue]{hyperref}
\usepackage[capitalize,noabbrev]{cleveref}

\def\1{\bm{1}}

\def\ind{\bm{1}}

\DeclareMathAlphabet{\mathsfit}{\encodingdefault}{\sfdefault}{m}{sl}
\SetMathAlphabet{\mathsfit}{bold}{\encodingdefault}{\sfdefault}{bx}{n}

\def\cA{{\mathcal{A}}}

\def\cN{{\mathcal{N}}}

\def\cS{{\mathcal{S}}}

\newcommand{\E}{\mathbb{E}}

\newcommand{\R}{\mathbb{R}}

\DeclareMathOperator*{\argmin}{arg,min}

\crefname{theorem}{theorem}{theorems}
\Crefname{theorem}{Theorem}{Theorems}
\crefname{lemma}{lemma}{lemmas}
\Crefname{lemma}{Lemma}{Lemmas}
\crefname{proposition}{proposition}{propositions}
\Crefname{proposition}{Proposition}{Propositions}
\crefname{corollary}{corollary}{corollaries}
\Crefname{corollary}{Corollary}{Corollaries}
\crefname{definition}{definition}{definitions}
\Crefname{definition}{Definition}{Definitions}
\crefname{assumption}{assumption}{assumptions}
\Crefname{assumption}{Assumption}{Assumptions}
\crefname{remark}{remark}{remarks}
\Crefname{remark}{Remark}{Remarks}
\crefname{example}{example}{examples}
\Crefname{example}{Example}{Examples}

\theoremstyle{plain}
\newtheorem{theorem}{Theorem}[section]

\newtheorem{lemma}[theorem]{Lemma}

\newtheorem{example}[theorem]{Example}
\theoremstyle{definition}
\newtheorem{definition}[theorem]{Definition}
\newtheorem{assumption}[theorem]{Assumption}
\theoremstyle{remark}
\newtheorem{remark}[theorem]{Remark}

%% file: abstract.tex
With advances in generative AI, decision-making agents can now dynamically create new actions during online learning, but action generation typically incurs costs that must be balanced against potential benefits. 
We study an online learning problem where an agent can generate new actions at any time step by paying a one-time cost, with these actions becoming permanently available for future use. 
The challenge lies in learning the optimal sequence of two-fold decisions: which action to take and when to generate new ones, further complicated by the triangular tradeoffs among exploitation, exploration and \emph{creation}.
To solve this problem, we propose a doubly-optimistic algorithm that employs Lower Confidence Bounds (LCB) for action selection and Upper Confidence Bounds (UCB) for action generation. 
Empirical evaluation on healthcare question-answering datasets demonstrates that our approach achieves favorable generation-quality tradeoffs compared to baseline strategies. 
From theoretical perspectives, we prove that our algorithm achieves the optimal regret of $O(T^{\frac{d}{d+2}}d^{\frac{d}{d+2}} + d\sqrt{T\log T})$, providing the first sublinear regret bound for online learning with expanding action spaces.


%% file: introduction.tex
Sequential decision-making problems involve agents repeatedly selecting actions from a candidate set to maximize cumulative reward. Traditional approaches assume a fixed set of available actions, focusing on the exploration-exploitation tradeoffs: balancing empirically high-reward actions (exploitation) against less-tested alternatives (exploration). However, advances in generative AI have introduced a new paradigm where contemporary systems can dynamically \emph{expand} their action spaces by \emph{creating} novel actions over time. This capability introduces an additional strategic dimension that agents should also balance immediate performance with strategic investments in future capabilities enabled by new actions. Consider the following motivating scenarios:

\begin{example}[Healthcare Question-Answering Systems]
\label{example:healthcare_q&a}
AI-powered healthcare platforms must decide between reusing existing vetted responses from their FAQ libraries or investing in creating new, tailored responses for novel patient inquiries. Each custom response requires costly expert review and validation (potentially hundreds of dollars when accounting for clinical expertise). However, once created and vetted, these responses become reusable assets. When a patient in a given region asks ``What are healthy meals during pregnancy?'', the system faces a critical choice: provide a generic response about pregnancy nutrition, or invest in creating a new response more specific to typical foods in that region, benefiting hundreds of future expectant mothers in similar settings.
\end{example}

\begin{example}[Personalized Advertisement]
\label{example:personalized_ad}
An advertising platform may initially start with a finite set of ad templates for different user contexts. Over time, the platform observes new user segments and decides to design specialized ads (with initial design and production costs) perfectly customized to the new user subgroups. Once created, these specialized ads become available for future targeting at no additional cost.
\end{example}

In both scenarios, the agent must decide at each time step whether to select an existing action or pay a one-time cost to instantiate a new action perfectly suited to the observed context. This introduces a novel \emph{create-to-reuse} problem that goes beyond traditional exploration-exploitation tradeoffs.

\paragraph{Problem Formulation (preview).} We study an online learning problem with an \emph{expanding} action space. At each time $t$, the agent first observes a context $x_t$. Then it may either
\begin{enumerate}[label=(\alph*)]
    \item Select an \emph{existing} action, incurring some (potentially suboptimal) loss, or
    \item \emph{Generate} a new action that is customized the current $x_t$ (without any excess loss), at a fixed one-time cost.
\end{enumerate}

This formulation has two key features. First, step (b) is notable in that the agent generates a new action only through an oracle $\cA(x_t)$ that is prompted by the context $x_t$. The learning algorithm operates as a decision-making layer on top of this custom action oracle. In contrast, traditional online learning or bandits operate directly in the action space. Second, once generated, new actions can be reused in future rounds without additional expense. The key is to judiciously decide when to pay the cost of adding such a specialized action and when to rely on existing arms.

This setting presents fundamental challenges that distinguish it from existing online learning and bandits frameworks. We face a triangular tradeoffs among three competing objectives: exploitation (using known good actions), exploration (learning about existing uncertain actions), and \textbf{creation} (generating a new action to satisfy immediate needs while enriching future capabilities). Additionally, we have no prior experience with potential new actions or unlimited freedom to generate arbitrary ones -- each creation must be specifically tailored to the current context at a fixed cost.

\paragraph{Summary of Contributions}
Our main contributions are fourfold:
\begin{enumerate}
\item \textbf{Problem Modeling}: We establish a new problem formulation that allows for costly expansion of the action space in online learning, formalizing the create-to-reuse framework. 
\item \textbf{Algorithmic Framework}: We propose a \emph{doubly-optimistic} algorithm that uses Lower Confidence Bounds (LCB) when selecting among existing actions, and Upper Confidence Bounds (UCB) when deciding whether to generate new actions. This design simultaneously exploits near-optimal actions and enables creation without excessive hesitation.
\item \textbf{Empirical Validation}: We conduct experiments on real-world healthcare question-answering datasets, demonstrating that our approach achieves favorable generation-quality tradeoffs compared to baselines. Our results show the method gracefully interpolates between pure reuse and always-create policies while maintaining superior performance.
\item \textbf{Optimal Regret Guarantees}: Under a semi-parametric loss model, our algorithm achieves $O(T^{\frac{d}{d+2}}d^{\frac{d}{d+2}} + d\sqrt{T\log T})$ expected regret, where $T$ is the time horizon and $d$ is the dimension of covariates. We prove this rate is optimal by establishing a matching $\Omega(T^{\frac{d}{d+2}})$ information-theoretic lower bound.
\end{enumerate}

\paragraph{Technical Novelty.}
The crux of our approach is a \textbf{double optimism} principle, which resolves the unique challenge of balancing creation with exploration/exploitation. Among existing actions, we rely on their \emph{LCB} comparisons to both exploit high-performing actions and continue exploring uncertain ones. When evaluating creation decisions, we compare the \emph{UCB} loss of the best existing action against the fixed generative cost, triggering creation with appropriate probability. This double optimism perspective naturally maximizes the long-term value of new actions while tightly controlling worst-case regret.

\paragraph{Paper Organization.} The rest of this paper is organized as follows. We discuss related literature in \Cref{sec:related_works}, then present the rigorous problem setting in \Cref{sec:problem_setup} along with the necessary assumptions. We propose our main algorithm in \Cref{sec:algorithm}, analyze its theoretical performance in \Cref{sec:regret_analysis}, and conduct numerical experiments to validate its real-world performance in \Cref{sec:empirical}.

%% file: related_works.tex
Here we discuss related literature on the most relevant topics in online decision making. Please refer to \Cref{sec:more_discussion} for discussions on broader fields including active learning, digital healthcare, recommendation system, and inventory management.


\textbf{Multi-Armed and Contextual Bandits.} The multi-armed bandit (MAB) problem has been extensively studied since \citet{lai1985asymptotically}. The classic framework \citep{auer2002finite,agarwal2014taming}, that a decision-maker repeatedly selects from a fixed set of arms, was extended to contextual bandits \citep{li2010contextual,chu2011contextual} where rewards depend on observable contexts. The crux is to balance exploration and exploitation with the goal of \emph{regret} minimization. Please refer to \citet{slivkins2019introduction} for a comprehensive discussion.

\textbf{Online Facility Location.}
Online facility location (OFL), studied by \citet{meyerson2001online}, \citet{fotakis2008competitive}, and \citet{guo2020facility}, is closely related to our formulation. In OFL, algorithms decide whether to open new facilities or assign requests to existing ones, minimizing facility costs plus assignment distances. While structurally similar to our problem, there are crucial differences. First, OFL assumes \emph{known} distance metrics, while we must learn \emph{unknown} parameters defining distances. Second, OFL \emph{automatically} assigns points to nearest facilities, while we must \emph{actively} select actions under uncertainty. Therefore, OFL involves a \emph{two-way} trade-offs between immediate costs and future benefits, whereas our problem requires a \emph{three-way} balance between exploitation, exploration, and creation, necessitating our novel algorithmic approach.

\textbf{Online Learning with Resource Constraints.} Another line of related research studies resource-limited bandits, such as ``bandits with knapsack (BwK)'' \citep{badanidiyuru2013bandits} and its versions \citep{agrawal2016linear,immorlica2019adversarial,liu2022non}. In these scenarios, each arm-pulling consumes some portion of a finite resource (e.g., budget, time, or capacity), and the algorithm aims to optimize the cumulative reward before resources run out. However, these approaches cannot be directly applied to our problem because of a key difference in resource consumption patterns. In BwK, resource consumption only affects the current period's decision-making. In contrast, our setting involves a one-time cost for creating new arms that provides benefits across all future periods through expansion of the action space. Besides, BwK mostly assumes a \emph{hard} constraint on budgets, while we adopt a \emph{soft} constraint as an additional cost in our problem setting.

%% file: problem_setup.tex
We now formalize the problem of creating-to-reuse as an online decision-making framework. In order to demonstrate the problem setting, we start with the healthcare Q\&A scenario described in Example~\ref{example:healthcare_q&a}. As an abstraction, each arriving patient question is represented as a $d$-dimensional \emph{context} vector $x_t$ in a learned semantic embedding space. 
The system maintains a context library $S_t$ of vetted FAQ entries, implemented as a \emph{hash table} where each context that has been previously added serves as a key to its corresponding custom respond (or generally the \emph{action}) generated by an oracle $\mathcal{A}(\cdot)$. Crucially, the algorithm operates only in the context representation space by searching through context keys in $S_t$. When a new question $x_t$ arrives, the algorithm makes decisions based on estimated losses and can either:

\begin{enumerate}[label=(\alph*)]
    \item Decide to create a new custom response by paying a fixed cost $c$ and adding context $x_t$ as a new key to the library. The generation oracle $\mathcal{A}(\cdot)$ then automatically produces the tailored action $a_t = \mathcal{A}(x_t)$, and the pair $(x_t, a_t)$ becomes permanently available for future reuse. \emph{Or}
    \item Select an existing context key $f \in S_t$ from the library. The system automatically retrieves the corresponding action $a_t = S_t(f) = \mathcal{A}(f)$ and deploys it for context $x_t$, incurring a mismatch loss $d(x_t, f)$ that reflects the difference between (1) the custom response to context $x_t$ versus (2) the action tailored for another context $f$.
\end{enumerate}

Technically, we consider the following problem setting.
\smallskip

\fbox{\parbox{0.95\textwidth}{
Initialization: Context-to-action oracle $\cA(\cdot)$. A library $S_1 =\{f, \mathcal{A}(f)\}$ with context keys $f$ and vetted custom actions $\mathcal{A}(f)$. \\ 
For $t=1,2,...,T:$
    	\noindent
    		\begin{enumerate}[leftmargin=*,align=left]
    			\setlength{\itemsep}{0pt}
                    \item Observe $x_t\in\R^d$ (patient question arrives).
    			\item The algorithm decides whether to create a customized response to $x_t$. If YES, then
                    \begin{enumerate}[label=(\roman*)]
                        \item Generation oracle produces and deploys $a_t = \mathcal{A}(x_t)$ (custom response to $x_t$).
                        \item Receive a fixed loss $c$ (creation cost).
                        \item Update $S_{t+1} := S_t \cup \{x_t: a_t\}$ (add new context-action pair to the library).
                    \end{enumerate}
                    \item If NO, then
                    \begin{enumerate}[label=(\roman*)]
                        \item Select an existing context key $f_t \in S_t$ and retrieve $a_t = S_t(f_t)$.
                        \item Receive a loss $l_t:=d(x_t, f_t) + N_t$ (noisy mismatch penalty).
                        \item Update $S_{t+1}:=S_t$ (library unchanged).
                    \end{enumerate}
    		\end{enumerate}
    	}
}
\smallskip

In this formulation, $d(x_t, f_t)$ captures the expected mismatch loss when deploying an action originally designed for context $f_t$ to serve context $x_t$. While this fundamentally reflects the difference between $\cA(x_t)$ and $\cA(f_t)$ in the action space, the algorithm can only estimate this through context-space relationships since it lacks direct access to $\cA(x_t)$ (actions having not been generated yet). 

For theoretical analysis, our main modeling assumption is that this mismatch can be captured by a \textit{squared distance} function in the context space. In experiments, we consider other forms for the mismatch distance. 

\begin{assumption}[Quadratic parametric loss]
    \label{assumption:distance}
    We assume the distance function satisfies
    \begin{equation}
        \label{eq:distance}
        d(x,f):=(x-f)^{\top}W(x-f)
    \end{equation}
    where $W\in\mathbb{S}_+^d$ is an \textbf{unknown} positive semi-definite $d\times d$ matrix. Accordingly, denote
    \begin{equation}
        \label{eq:notation_w_phi}
        \begin{aligned}
            w:=& Vec(W)\in\R^{d^2}\\
            \phi(x,f):=& Vec[(x-f)(x-f)^\top]\in\R^{d^2},
        \end{aligned}
    \end{equation}
    and we have an equivalent definition as $d(x,f):=\phi(x,f)^\top w$.
\end{assumption}

\textbf{Why we assume a quadratic parametric loss?} The motivation is that contexts are embedded in a space where different dimensions capture semantically relevant information. The cost of reusing an action designed for one context when serving another can be modeled as a distance on this representation space, although the exact importance weighting of different semantic dimensions (captured by matrix $W$) is unknown to the learner. Since $d(x_t, x_t) = 0$, our formulation measures the \textit{excess} cost due to not generating a custom action for each $x_t$. This fits scenarios where the algorithm interacts with complex action spaces through oracle $\mathcal{A}(x_t)$ (human expert or generative model); our aim is to achieve good performance relative to this oracle's capabilities. Modeling $d(\cdot, \cdot)$ as a squared distance function captures more structure than linear parametric choices while remaining more tractable than nonparametric formulations. Furthermore, the empirical results our algorithm performs on real-world Healthcare Q\&A datasets validate the robustness of this modeling.

\textbf{Goal of Algorithm Design.} Our goal is to minimize the expected \emph{total loss}. We will rigorously define the performance metric and technical assumptions at the beginning of \Cref{sec:regret_analysis}.

%% file: algorithm.tex
To solve this online decision-making problem with expanding context libraries, we propose a ``Doubly-Optimistic'' algorithm. This section presents the algorithm design and highlights its properties. We will analyze and bound its cumulative regret in the next section.

\begin{algorithm}[htbp]
    \caption{Doubly-Optimistic Algorithm}
    \label{algo:doubly_opt}
    \begin{algorithmic}[1]
        \STATE {\bfseries Initialization:} \textbf{Custom action oracle} $\cA(x)$, and $\Sigma_0 = \lambda\cdot I_{d^2}, b_0 = \vec{0}, S_1 = \{\vec{\ind}_d: \cA(\vec{\ind}_d)\}$, $\alpha$.
        \FOR{$t=1,2,\ldots, T$}
            \STATE Observe context $x_t\in\R^d$
            \FOR{$\forall f\in S_t$ (all existing context keys)}
                \STATE \textbf{Compute loss estimates and uncertainties}. Denote
                \begin{equation}
                    \label{eq:denotes}
                    \begin{alignedat}{2}
                        \Delta_t(x, f):&=\alpha\cdot\sqrt{\phi(x_t, f)^\top\Sigma_{t-1}^{-1}\phi(x_t, f)},\  
                        &\bar{d}_t(x,f):=\phi(x,f)^\top\Sigma_{t-1}^{-1}b_{t-1}\\
                        \hat{d}_t(x,f):&=\bar{d}_t(x,f) + \Delta_t(x,f),\qquad
                        &\check{d}_t(x,f):=\bar{d}_t(x,f) - \Delta_t(x,f)\\
                    \end{alignedat}
                \end{equation}
            \ENDFOR
            \STATE Select context key $f_t:=\argmin{f\in S_t} \check d_t(x_t, f)$.
            \IF{{ $Z_t == 1$} with $Z_t\sim Ber(\min\{1, \frac1c\cdot\hat d_t(x_t, f_t)\})$ as an i.i.d. Bernoulli random variable}
                \STATE \textbf{Decide to create new}: Oracle generates action $a_t=\cA(x_t)$ and deploys it at a cost $c$.
                \STATE Receive loss $l_t = 0$ (perfect match for custom action).
                \STATE Update context library $S_{t+1} = S_t\cup{x_t: a_t}$ (add new context-action pair).
                \STATE Keep $\Sigma_t:=\Sigma_{t-1}$ and $b_t:= b_{t-1}$ without updating.
            \ELSE
                \STATE \textbf{Decide to reuse}: Retrieve and deploy action $a_t = S_t[f_t] = \cA(f_t)$ at no creation cost.
                \STATE Receive mismatching loss $l_t = d(x_t, f_t) + N_t$.
                \STATE Maintain context library $S_{t+1} = S_t$ (no new entries).
                \STATE Update loss estimation parameters
                \begin{equation}
                    \label{eq:update_parameters}
                    \begin{aligned}
                        \Sigma_t :&= \Sigma_{t-1}+\phi(x_t, f_t)\phi(x_t, f_t)^\top,\ 
                        b_t :=b_{t-1} + l_t\cdot\phi(x_t, f_t).
                    \end{aligned}
                \end{equation}
            \ENDIF
        \ENDFOR
    \end{algorithmic}
\end{algorithm}

The pseudocode of our algorithm is displayed in \Cref{algo:doubly_opt}. At each time $t$, the algorithm inherits loss estimation parameters $\Sigma_{t-1}$, $b_{t-1}$ \footnote{Linear regression parameters. We estimate the vector $w$ as $\Sigma_{t-1}^{-1}b_{t-1}$ at every time step $t$.} and context library $S_t$ from $(t-1)$, then observes a new context vector $x_t$. Using the estimation parameters, it computes predicted mismatch loss $\bar{d}_t(x_t, f)$ and uncertainty bound $\Delta_t(x_t, f)$ for each existing context key $f\in S_t$. The algorithm operates entirely in the context representation space and takes the following two steps to determine which action to deploy.

\begin{enumerate}[label=(\roman*)]
    \item \textbf{Lower Confidence Bound (LCB) loss on existing contexts.} For each existing context key $f$, we calculate the LCB loss as $\check{d}_t(x_t, f)=\bar{d}_t(x_t, f)-\Delta_t(x_t, f)$. We then select $f_t$ as the context key with the minimum LCB loss. Note that we do not immediately choose to reuse this context.
    \item \textbf{Upper Confidence Bound (UCB) probability to create a new entry.} After identifying $f_t$ as the best existing option, we turn to consider its UCB loss $\hat d_t(x_t, f_t)=\bar{d}_t(x_t, f_t) + \Delta_t(x_t, f_t)$ and compare it against the fixed creation cost $c$. With a probability of $\min\{1, \frac{\hat d_t(x_t, f_t)}c\}$, we decide to create a new entry: The oracle generates $a_t=\cA(x_t)$ and we add context $x_t$ to the library. Otherwise, we reuse the existing context $f_t$, and the system retrieves $a_t = S_t[f_t]$ from library $S_t$. After receiving a mismatching loss $l_t$, the algorithm update the estimation parameters $\Sigma_t$ and $b_t$ accordingly.
\end{enumerate}

As the argmin of LCB loss, $f_t$ represents the existing context that could potentially yield the lowest mismatch under optimistic assumptions, balancing exploration and exploitation given historical uncertainties. This approach aligns with contextual bandit methods such as \citet{chu2011contextual}.

The UCB-based creation probability $\frac{\hat d_t(x_t, f_t)}c$ increases the chance of adding new contexts when the estimated mismatching loss is high relative to creation cost (within a risk $\Delta_t$ we can tolerate). This design enables us to estimate the ``necessity'' of creating new entries while bounding the total expected loss accumulated \emph{before} any new context is added to a particular region of the context space. We explain this property later in \Cref{lemma:constant_loss_bound_before_a_new_arm_emerges}.

\paragraph{Computational Complexity}
\Cref{algo:doubly_opt} incurs worst-case time complexity $O(d^4T^2)$, as it compute matrix-vector products of $d^2$ dimension for every context key $f\in S_t$ at each round $t$, with at most $T$ contexts possible. Since the expected number of newly created contexts is $O(T^{\frac{d}{d+2}})$ with respect to $T$, the \emph{expected} complexity refines to $O(T^{\frac{2d+2}{d+2}})$.
Given that $d$ can represent sentence embedding dimensions in the Q\&A scenario, an $O(d^4)$ time complexity is impractical. In our real-data numerical experiments, we improve computational performance by replacing the estimated distance function $\bar{d}_t(x,f)$ with either of the two forms:

\begin{enumerate}[label=(\alph*)]
    \item A squared linear model $(\theta^\top (x-f))^2$ (equivalent to setting $W=\theta\theta^\top$), reducing the computational complexity to $O(d^2)$. Uncertainty bounds are derived from ridge regression covariance matrices.
    \item A neural network $d(x,f;\Theta)$ with uncertainty function $\Delta_t(x,f;\Theta)$ derived under Gaussian conjugate assumptions, reducing complexity to $O(D^2)$ where $D:=\|\Theta\|_0$ is the number of NN parameters.
\end{enumerate}

On the other hand, we implement the original algorithm in the synthetic-data simulations to validate the theoretic guarantees with respect to $T$ (for small $d$'s only).

%% file: theoretical_analysis.tex
In this section, we provide a regret analysis of our algorithm.
We first state the performance metric and necessary technical assumptions. Then we present the main theorem on the algorithmic regret upper bound. Finally, we provide a corresponding lower bound that matches the leading term of the upper bound with respect to $T$.

\subsection{Definitions and Assumptions}
\label{subsec:assumption}
As we have stated by the end of \Cref{sec:problem_setup}, our goal is to minimize the \emph{total loss}. In order to measure the performance, we adopt the expected \emph{regret} as the loss metric, which is defined as follows:
\begin{definition}[Optimal and Regret]
    Denote the minimal expected loss\footnote{Expectation taken over observation noises only. Same for the definition of $OPT_o$.} that is \emph{achievable} in hindsight as $OPT_h$, which equals:
    \begin{equation}
        \label{eq:opt_h}
        \begin{aligned}
            OPT_h:=\min_{\cS:=\{S_1, S_2, \ldots, S_T, S_{T+1}|S_{t+1} \setminus S_t \subseteq\{x_t\}\}} c\cdot|S_{T+1}| + \sum_{t=1}^T \min_{f\in S_{t+1}} d(x_t, f). 
        \end{aligned}
    \end{equation}
    There also exists a \emph{non-achievable} minimal loss denoted as $OPT_o$, which is only accessible by an omniscient oracle that knows $\{x_t\}_{t=1}^T$ and selects an optimal option set ahead of time:
    \begin{equation}
        \label{eq:opt_o}
        \begin{aligned}
            OPT_o:=\min_{S} c\cdot |S| + \sum_{t=1}^T\min_{f\in S} d(x_t, f).
        \end{aligned}
    \end{equation}
    From the definition, we know that $OPT_o \leq OPT_h$. Also, denote the expected loss obtained by our algorithm as $ALG$, which equals:
    \begin{equation}
        \label{eq:alg}
        \begin{aligned}
            ALG:=c\cdot |S_{T+1}| + \sum_{t=1}^T \min_{f\in S_{t+1}} d(x_t, f).
        \end{aligned}
    \end{equation}
    Define the regret $REG$ as the expected loss difference\footnote{Expectation taken over the $\{x_t\}_{t=1}^T$ series.} between $OPT_h$ and $ALG$.
    \begin{equation}
        \label{eq:reg}
        \begin{aligned}
            REG:=\E[ALG-OPT_h]
        \end{aligned}
    \end{equation}
\end{definition}

We then make two distributional assumptions on the covariates and the noises, respectively.
\begin{assumption}[Covariate distribution and norm bound]
    \label{assumption:covariate}
    Assume $x_t\in\R^d, t=1,2,\ldots, T$ are drawn from independent and identical distributions (i.i.d.), with $d\geq 2$. Also, assume a norm bound as $\|x_t\|_2\leq 1$.  
\end{assumption}
Assumption~\ref{assumption:covariate} is necessary for us to effectively learn the metric matrix $W$ through online linear regression. For the same reason, we assume a subGaussian noise on the observations as follows:

\begin{assumption}[Noise distribution]
    Assume that $N_t\in\R, t=1,2,\ldots, T$ are drawn from \emph{$\sigma$-subGaussian i.i.d.}, where $\sigma$ is a universal constant.
\end{assumption}

\subsection{Regret Bounds}
In this subsection, we sequentially present our theoretical guarantees on the regret upper and lower bounds, as the following two theorems.

\begin{theorem}[Regret upper bound]
    \label{theorem:regret}
    With assumptions made in \Cref{subsec:assumption}, the expected regret of our \Cref{algo:doubly_opt} is upper bounded by $O(T^{\frac{d}{d+2}}d^{\frac{d}{d+2}} + d\sqrt{T\log T})$.
\end{theorem}

\begin{proof}[Proof Sketch]
We prove \Cref{theorem:regret} in the following sequence:
\begin{enumerate}
    \item (\Cref{lemma:opt_o_upper_bound}) We upper bound the non-achievable minimal loss as $OPT_o=O(T^{\frac{d}{d+2}}d^{\frac{d}{d+2}})$. This is proved by a fine-grid covering of the space.
    \item (\Cref{lemma:competitive_ratio_upper_bound}) We upper bound the algorithmic loss $ALG$ within a constant competitive ratio of $OPT_o$ adding cumulative prediction errors: $\E[ALG]=O(\E[OPT_o + \sum_{t=1}^T\Delta_t(x_t, f_t)])$. To prove this, we divide $\{x_t\}$'s into ``good'' and ``bad'' groups, and bound their excess loss respectively.
    \item (\Cref{lemma:linear_regression}) We upper bound the excess risk $\E[\sum_{t=1}^T\Delta_t(x_t, f_t)]= O(d\sqrt{T\log T})$ by standard online linear regression (similar to \citet{chu2011contextual} by replacing $d$ with $d^2$).
    \item Finally, we derive the regret upper bound as $REG=\E[ALG-OPT_h]=O(T^{\frac{d}{d+2}}d^{\frac{d}{d+2}} + d\sqrt{T\log T})$ according to the three steps above.
\end{enumerate}

Please refer to \Cref{appendix:sec_proof_details} for all technical details of this proof, including rigorous statements of lemmas and derivations of inequalities.
\end{proof}
To show the optimality of the regret upper bound proposed above, we present the information-theoretic lower regret bound.

\begin{theorem}[Regret lower bound]
    \label{theorem:lower_bound}
    For any online learning algorithm, there exists an instance of problem setting presented in \Cref{sec:problem_setup}, such that the regret is at least $\Omega(T^{\frac{d}{d+2}})$ with respect to $T$ (despite the dependence on $d$).
\end{theorem}
We defer the proof to \Cref{appendix:proof_of_lower_bound}. The main idea is to apply the $\Omega(K^{-\frac{2}{d}})$ lower bound for the K-nearest-neighbors (K-NN) problem, along with an optimal choice of $K$ that balance this term with $c\cdot K$. \Cref{theorem:lower_bound} indicates that our algorithm achieves an optimal regret with respect to $T$.

%% file: empirical.tex
In this section, we conduct numerical experiments to validate our method's performance. We first run the original algorithm on low-dimensional synthetic data to demonstrate the regret dependence on $T$. Then we adapt our algorithm to real-world healthcare Q\&A scenarios and show better tradeoffs between generation cost and mismatching loss compared to baselines.

\subsection{Regret Validation on Synthetic Data}
\begin{figure}[t]
    \centering
    \begin{subfigure}[b]{0.32\textwidth}
        \centering
        \includegraphics[width=\linewidth]{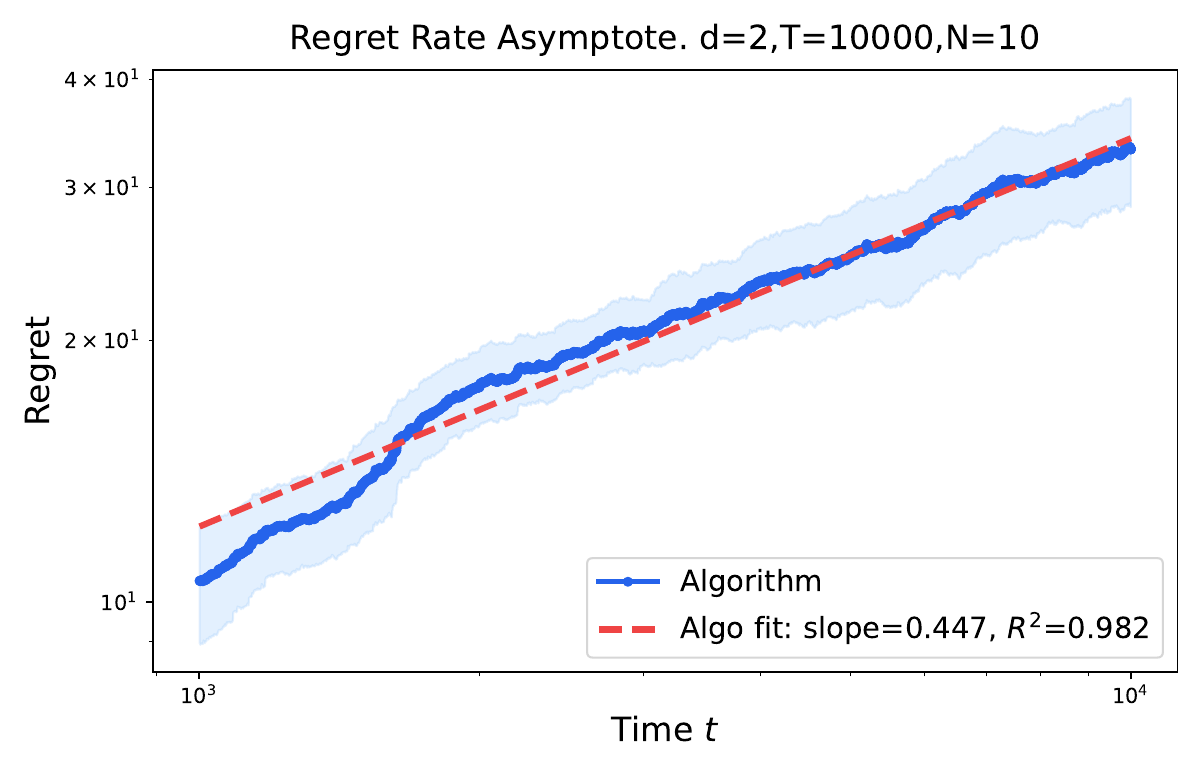}
        \caption{$d=2$}
        \label{subfig:d=2}
    \end{subfigure}
    \hfill
    \begin{subfigure}[b]{0.32\textwidth}
        \centering
        \includegraphics[width=\linewidth]{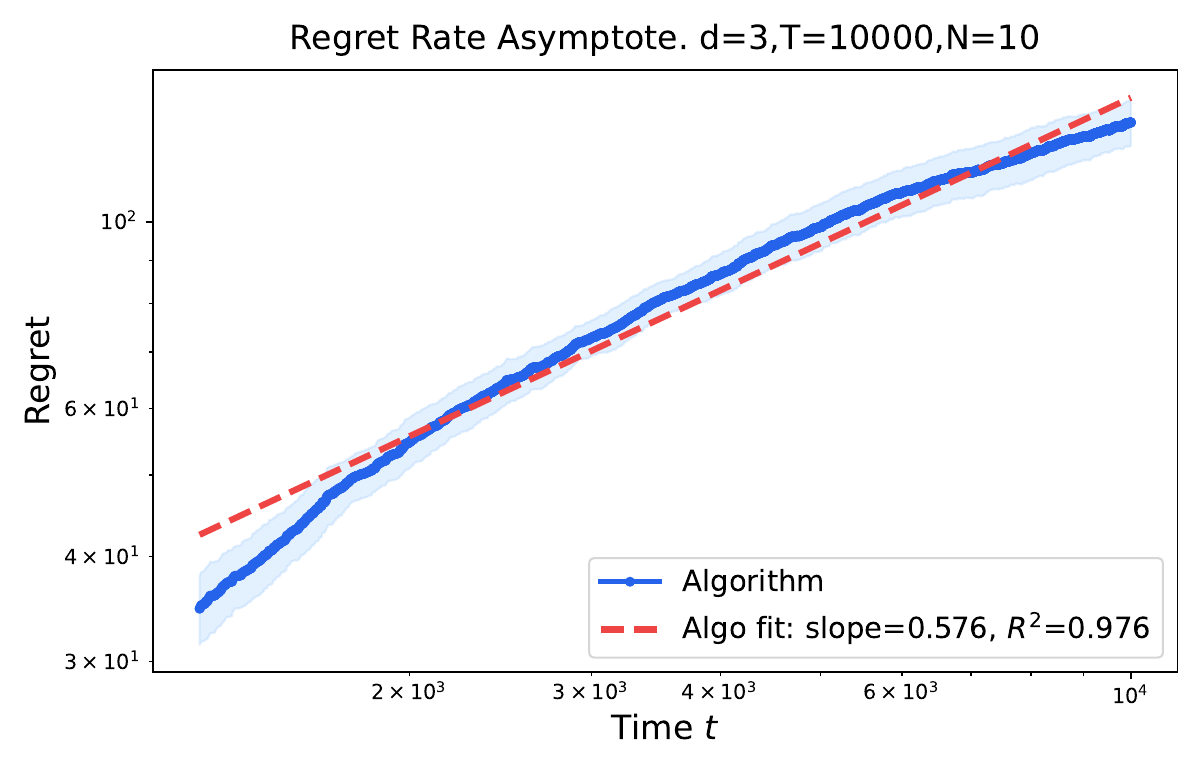}
        \caption{$d=3$}
        \label{subfig:d=3}
    \end{subfigure}
    \hfill
    \begin{subfigure}[b]{0.32\textwidth}
        \centering
        \includegraphics[width=\linewidth]{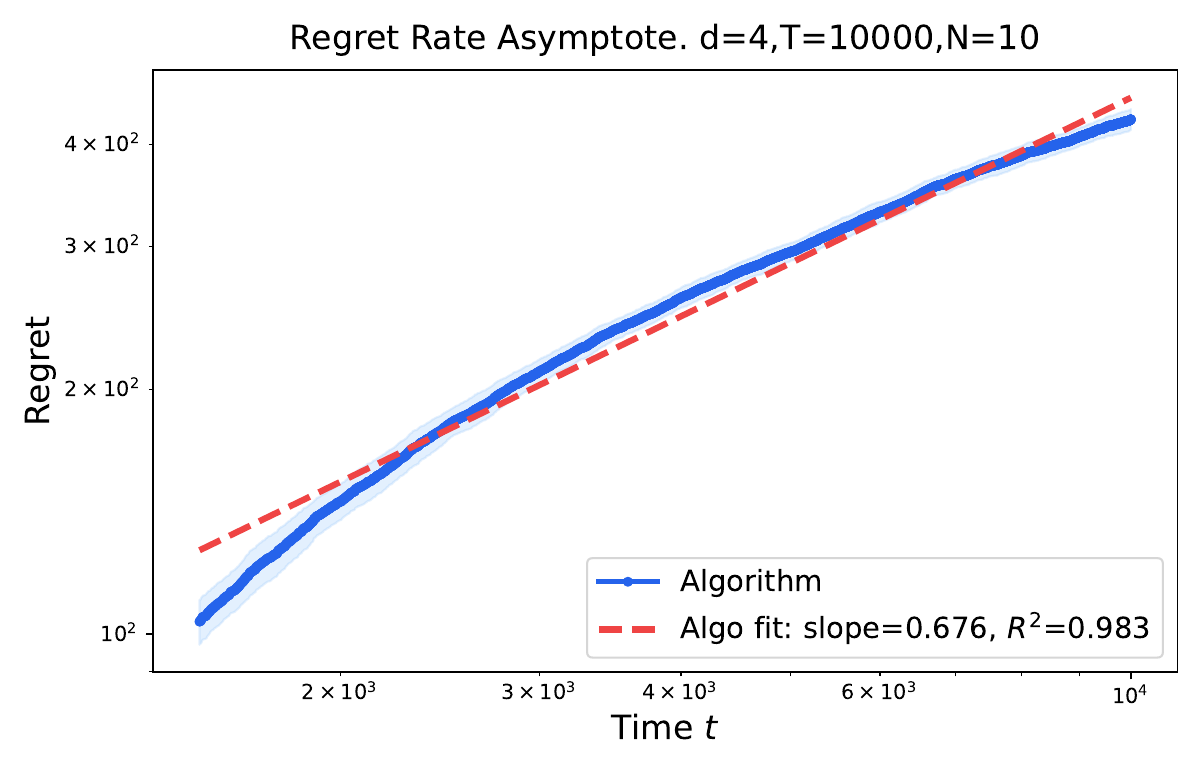}
        \caption{$d=4$}
        \label{subfig:d=4}
    \end{subfigure}
    \caption{Regret curves for $T=10000$ and $d=2,3,4$ in log-log scales, repeated by $N=10$ epochs. The slope of the linear asymptote under log-log diagram indicates the power dependence of regret on $T$, which should be $\frac{d}{d+2}$.} 
    \label{fig:regret_curves_loglog_scales}
\end{figure}

We evaluate our doubly-optimistic algorithm on synthetic data with dimensions $d=2,3,4$ over time horizon $T=10,000$, repeated for $N=10$ epochs. Context vectors $x_t$ are drawn from $L_2$-normalized uniform distributions, with noise $N_t\sim\cN(0, 0.05)$. We calculate regret by comparing the algorithmic loss against $OPT_o$ (defined in \Cref{eq:opt_o}), approximated by randomized K-means++ with Lloyd iterations over potentially optimal values of $K$. We do not apply $OPT_h$ as its computational cost is exponentially dependent on $T$.

\Cref{fig:regret_curves_loglog_scales} presents the regret curves in log-log scale to reveal the power dependence of regret on $T$. Our method exhibits empirical slopes of $0.447, 0.576, 0.676$ for $d=2,3,4$ respectively, aligning closely with the theoretical rates which should be $\frac{d}{d+2}$ according to \Cref{theorem:regret}. These results validate our theoretical analysis in synthetic environments.  

Note: We restrict experiments to low-dimensional settings due to the computational cost of $OPT_o$ (a necessary component of regret) in high dimensions, where K-means++ becomes ineffective and the underlying nearest neighbor problem is NP-hard. Despite these computational limitations, the synthetic validation confirms that our approach achieves the predicted theoretical regret rates, providing confidence in its performance for moderate-dimensional real-world applications.

\subsection{Generation-Quality Tradeoffs Analysis on Healthcare Q\&A Datasets}
\begin{figure}[t]
    \centering
    \begin{subfigure}[t]{0.4\textwidth}
        \centering
        \includegraphics[width=\linewidth]{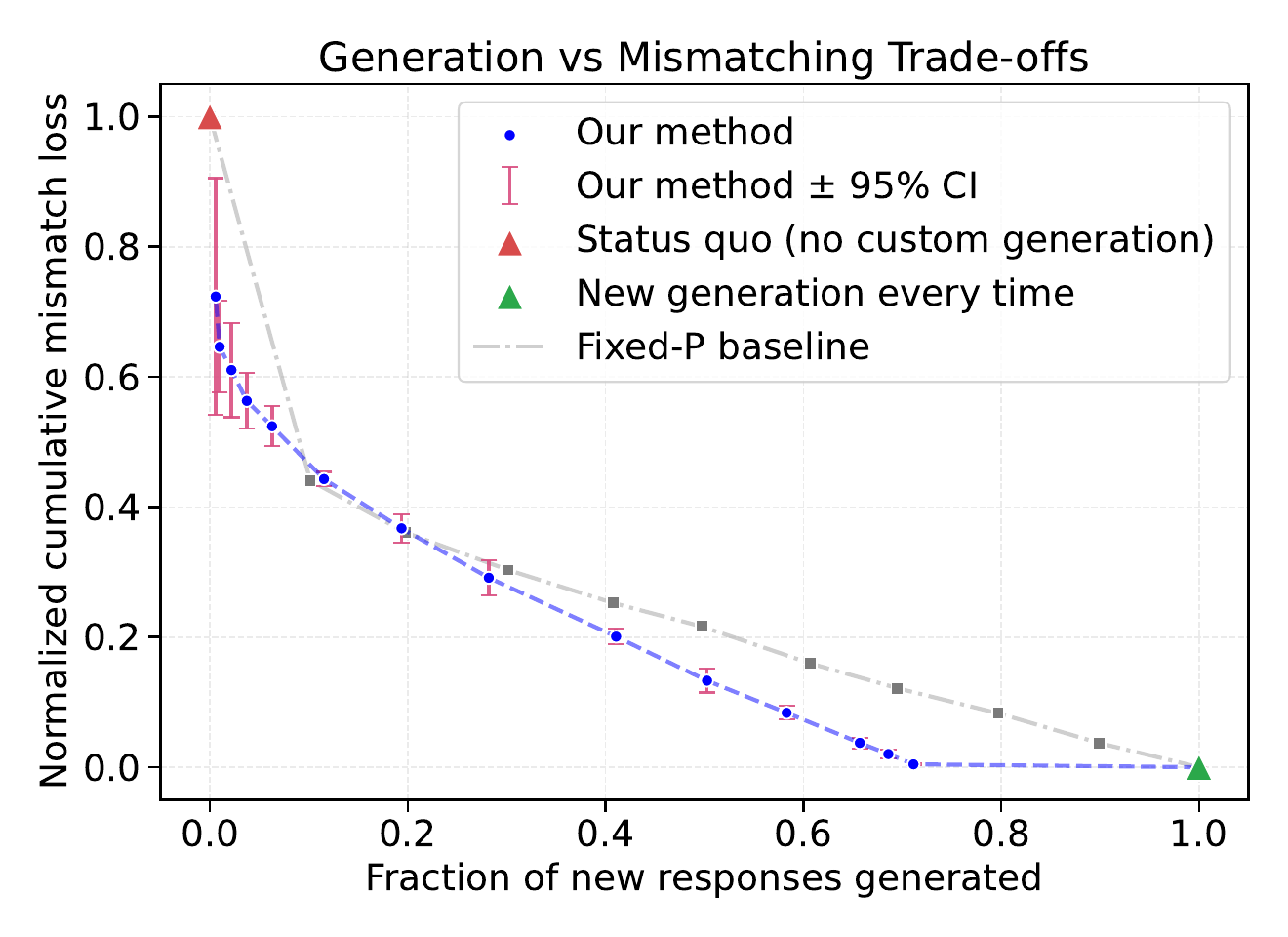}
        \caption{Numerical results on a maternal health Q\&A dataset from  \texttt{Nivi.Inc}.}
        \label{subfig:nivi}
    \end{subfigure}
    \begin{subfigure}[t]{0.4\textwidth}
        \centering
        \includegraphics[width=\linewidth]{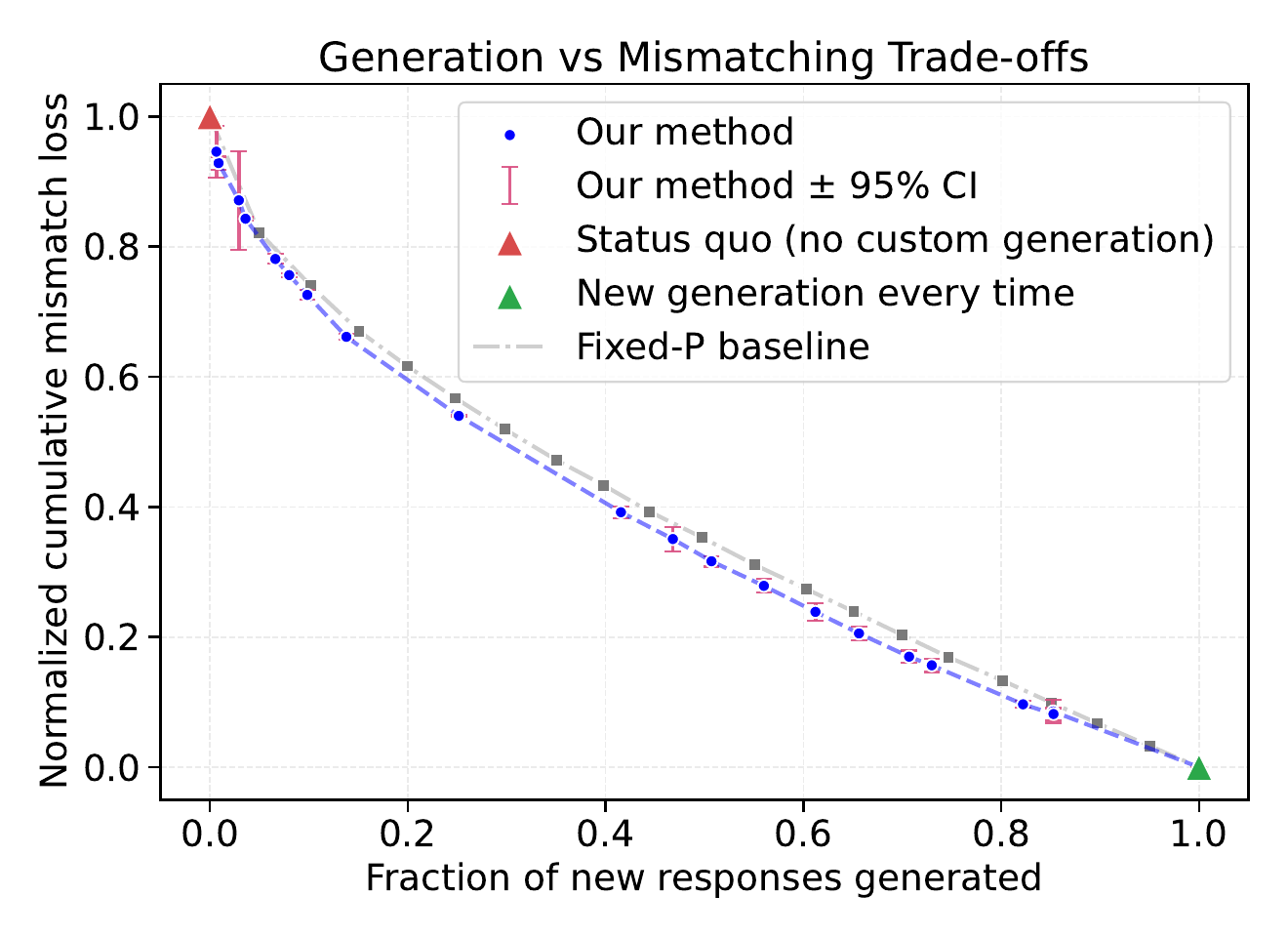}
        \caption{Numerical results on the public \quad\quad \href{https://www.kaggle.com/datasets/gvaldenebro/cancer-q-and-a-dataset}{Medical Q\&A Dataset}.}
        \label{subfig:medical_q_n_a}
    \end{subfigure}
    \caption{
    Tradeoffs between normalized generation costs (x-axis) and normalized mismatching loss (y-axis) on two healthcare Q\&A datasets. A lower/left curve indicates a better performance. Each blue point represents the (generation cost, mismatching loss) pair caused by a choice of $c$. In both cases, our algorithm outperforms the baseline that randomly generates custom responses with a variety of fixed probabilities $p$ (each gray point represents a choice of $p$).}
    \label{fig:q_a_experiments}
\end{figure}

We evaluate our algorithm on two real-world healthcare Q\&A datasets to demonstrate its practical effectiveness:

\begin{enumerate}
    \item \textbf{\texttt{Nivi}'s Maternal Health Dataset}: A dataset containing 839 user queries, with 12 pre-written FAQs for pregnant women, provided by \texttt{Nivi.Inc}, a company that provides healthcare chatbot services on WhatsApp.
    \item \textbf{Medical Q\&A Dataset}: A public collection of 47,457 medical question-answer pairs curated from 12 NIH websites (\url{https://www.kaggle.com/datasets/gvaldenebro/cancer-q-and-a-dataset}).
\end{enumerate}

\paragraph{Experimental Setup.}

Our experimental framework models the create-to-reuse decision process operating entirely in the context representation space. All questions are mapped to embeddings using OpenAI's pre-trained \texttt{text-embedding-3-small} model, creating a semantic representation space where the algorithm makes decisions. For each arriving question context $x_t$, the algorithm decides whether to select an existing context key $f$ from the FAQ library or add $x_t$ as a new context key and then invoke the custom answer generation oracle $\cA(\cdot)$.

Custom answer generation differs across datasets to reflect their nature. For the maternal health dataset, custom answers are generated by GPT-5 with carefully designed prompts including safety guardrails and emergency detection protocols appropriate for healthcare contexts. For the Medical Q\&A dataset, custom answers are directly retrieved from the pre-existing responses associated with each question entry.

Crucially, the mismatch loss feedback occurs in the action space rather than the context space. For current question context $x_t$ and an existing context $f$ in the FAQ library, the loss is calculated as $(1 - \text{cosine similarity})$ between \textbf{$x_t$'s custom answer} and \textbf{$f$'s custom answer}. This reflects our core assumption that the algorithm operates in context space while true loss manifests in action space, accessible only through the generation oracle $\cA(\cdot)$.

As we also mentioned in \Cref{sec:algorithm}, to maintain computational tractability, we model the estimated loss function as $\bar{d}(x,f) = (\theta^{\top}(x-f))^2$ (on the maternal health dataset) or adopt a neural network $d(x,f;\Theta)$ (on Medical Q\&A Dataset). 

We evaluate our doubly-optimistic algorithm against a fixed-probability baseline strategy. This baseline makes i.i.d. Bernoulli decisions $\sim \text{Ber}(p)$ at each time step: with probability $p$, generate a custom response; otherwise, select the most similar existing context from the library based on cosine similarity between question embeddings (note that it has no access to the custom answer before generation). To comprehensively evaluate performance across different cost-accuracy preferences, we vary the probability parameter $p$ uniformly across $[0,1]$ for the baseline. Meanwhile, we also vary the creation cost parameter $c$ from $0$ to $100$ for our algorithm, generating complete tradeoff curves for both approaches.

The numerical results are depicted in \Cref{fig:q_a_experiments}, where points and curves closer to the bottom-left indicate superior performance. We plot cumulative generation costs against cumulative mismatch losses, with both metrics normalized separately to $[0,1]$ scale for interpretability. Generation costs are normalized by the total cost of the always-generate strategy, while mismatch losses are normalized by the \emph{status quo} strategy that never generates custom responses. Note that these represent the two components of total loss in our formulation, depicted separately for clearer analysis. Each gray point represents a different choice of $p$ for the baseline, forming a curve that represents the best possible performance achievable by any fixed-probability strategy. Each experiment runs $N=10$ epochs with 95\% confidence intervals computed using Wald's test.



\paragraph{Results on \texttt{Nivi}'s Maternal Health Dataset.} \Cref{subfig:nivi} presents the generation-quality tradeoffs. Starting with 12 pre-written FAQs, our algorithm demonstrates several key advantages:

\begin{enumerate}
    \item \textbf{Context Clustering}: Compared with the always-generating strategy (green triangle), approximately 30\% of user questions exhibit sufficient similarity to existing FAQs, as evidenced by the algorithm achieving near-zero mismatch loss when generating responses for 70\% of queries.
    \item \textbf{Efficiency Gains}: Compared with \emph{status quo} (red triangle), strategic addition of just a few targeted FAQs reduces mismatch loss by approximately 25\% (as evidenced by the algorithmic curve approaching the point $(0, 0.75)$), highlighting the value of adaptive creation decisions over static policies.
    \item \textbf{Pareto Optimality}: Our algorithm consistently outperforms fixed-probability baselines throughout the entire generation spectrum, with statistical significance demonstrated by 95\% confidence intervals. The doubly-optimistic approach effectively pushes the performance frontier toward Pareto optimality.
\end{enumerate}

\paragraph{Results on Medical Q\&A Dataset.} 
\Cref{subfig:medical_q_n_a} presents results on the public Medical Q\&A dataset. We establish the initial FAQ library by prompting GPT-5 to classify all questions into 32 categories by topic, then randomly sampling 10 question-answer pairs from each category to create a generic response (a total of 32 FAQs).

Compared with the always-generating and FAQ-only baselines respectively, our algorithm can reduce about $60\%$ generation cost and about $60\%$ mismatch loss, leading to positive-sum tradeoffs (indicated by the convex curve).
Also, it achieves statistically significant improvements over fixed-probability baselines across nearly the entire generation spectrum, as confirmed by 95\% confidence intervals. 
However, the performance gains are notably smaller than those observed on the private maternal health dataset. We attribute this difference to the greater diversity in the Medical Q\&A dataset, spanning 37 question types across 32 medical topics. 
In contrast, \texttt{Nivi}'s dataset focuses specifically on maternal health with more concentrated topics and frequently recurring keywords, producing clearer semantic connections and stronger correlations between context and action similarities that enable more effective learning.

The results validate our theoretical framework in practice, demonstrating that principled confidence bound approaches for creation decisions significantly outperform heuristic alternatives in real-world healthcare applications where both response quality and resource efficiency are critical.

%% file: discussion.tex
\paragraph{Dynamic and Context-Dependent Creation Costs.} Our current framework assumes a fixed creation cost $c$ across all time steps and contexts. A natural extension would allow time-varying costs $c_t$ or context-dependent costs $c(x_t)$ that reflect realistic scenarios where creation difficulty varies with problem complexity or resource availability. This generalization would better capture applications like drug discovery, where synthesis costs depend on molecular complexity, or content generation, where review costs vary with topic sensitivity. However, this extension introduces significant algorithmic challenges, as evidenced by the substantially worse competitive ratios in variant-cost online facility location problems, where even achieving constant competitive ratios becomes impossible under adversarial sequences.

\paragraph{Non-Parametric and Neural Function Approximation.} While our theoretical analysis focuses on parametric quadratic loss functions $d(x,f)$, our empirical experiments demonstrate promising results when replacing the distance function with neural networks and using LLM-as-a-judge for feedback evaluation. Extending the theoretical guarantees to broader function classes, particularly neural networks or kernel methods, would significantly broaden the applicability of our framework. The key challenge lies in controlling the complexity of the function class while maintaining meaningful regret bounds, potentially requiring techniques from neural tangent kernels (NTK) or Bayesian optimization (BO) to handle the high-dimensional hypothesis space.

%% file: conclusion.tex
In this paper, we introduced an online decision-making problem where new actions can be generated on the fly, at a fixed cost, and then reused indefinitely. To address the balance among exploitation, exploration, and creation, we proposed a doubly-optimistic algorithm that achieves $O(T^{\frac{d}{d+2}}d^{\frac{d}{d+2}} + d\sqrt{T\log T})$ regret. This regret rate was proved optimal with a matching lower bound, and was validated through simulations. We also implemented our algorithm on a real-world healthcare Q\&A dataset to make decisions on generating new answers v.s. applying an FAQ. Our results open up new avenues for optimizing creation decisions in online learning, with potential extensions to broader loss models and flexible creation costs.

%% file: appendix.tex
{\huge Appendix}


\section{More Discussions}\label{sec:more_discussion}
In this section, we further discuss a few fields and topics of research that are related to our problem modeling, motivation, methodology and implementation.

\paragraph{Active Learning} Active learning frameworks fundamentally embody the exploration-exploitation-creation paradigm by allowing algorithms to strategically choose their training data, thereby naturally connecting to sequential decision-making with expanding action spaces. \citet{settles2009active} established the theoretical foundations for query selection strategies, while membership query synthesis approaches \citep{angluin1988queries} demonstrated how active learners can create entirely new query types rather than merely selecting from existing unlabeled data pools. Query-by-Committee methods \citep{seung1992query} and extended through frameworks like QUIRE by \citet{huang2010active} show how multiple learning strategies can be combined to create adaptive query selection policies that balance informativeness and representativeness. Closer work on meta-active learning and the ``Growing Action Spaces'' framework by \citet{farquhar2020growing} directly address expanding action spaces through curriculum learning approaches that progressively grow query complexity. The create-to-reuse framework maps directly onto active learning's core mechanisms: systems invest computational effort in synthesizing new query types, developing committee-based strategies, and learning meta-policies for query selection, creating reusable query generation mechanisms and adaptive selection strategies that can be applied across different datasets, domains, and learning tasks, while continuously expanding their query capabilities as they encounter new data distributions and learning scenarios.

\paragraph{Exploratory Learning for Unknown Unknowns.} Another notable progress is the exploratory machine learning (ExML) framework~\citep{zhao2024exploratory}. The authors introduced a novel and insightful approach to address unexpected unknown unknowns by exploring additional feature information through environmental interactions within a budget constraint, where an optimal bandit identification strategy is proposed to guide the feature exploration. There are several follow-up developments~\citep{kothawade2022active,rajendran2023unsupervised}. Compared to our create-to-use framework, there are two main differences: On the one hand, their work addresses the strategic choices of ``create'' while the current exploratory decisions would not be ``in use'' of future decisions. On the other hand, their cost serves as a budget consumption instead of a tradeoffs with cumulative utilities, analogous to the divergence between regret minimization and best-arm identification.

\paragraph{Digital Healthcare and Clinical Decision Support}
Digital healthcare and clinical decision support systems (CDSS) represent a rapidly evolving field where AI-powered systems must continuously balance the utilization of established medical knowledge with the creation of novel, patient-specific treatment protocols. Foundational work by \citet{rajpurkar2022ai} on diagnostic AI systems and the comprehensive framework established by \citet{moor2023foundation} demonstrate how modern medical AI systems expand beyond narrow, single-task applications to flexible models capable of diverse medical reasoning tasks. Reinforcement learning approaches in critical care, particularly the systematic review by \citet{liu2017deep} covering 21 RL applications in intensive care units, illustrate how these systems extend from discrete medication dosing decisions to continuous, multi-dimensional treatment optimization spaces. The create-to-reuse paradigm is particularly evident in precision medicine applications, where systems invest computational resources in developing personalized treatment protocols that can subsequently be applied to patients with similar phenotypic characteristics, effectively creating reusable clinical knowledge that scales across patient populations while maintaining individualized care quality.
\paragraph{Recommendation Systems and Personalization}
Recommendation systems research has evolved from static collaborative filtering approaches to sophisticated frameworks that dynamically balance the exploitation of existing user preferences with the creation of new personalized recommendation strategies. Neural Collaborative Filtering by \citet{he2017neural} and the Wide \& Deep Learning framework by \citet{cheng2016wide} established the foundation for deep learning approaches that can capture complex user-item interactions beyond traditional matrix factorization methods. Meta-learning approaches, particularly by \citet{lee2019melu} demonstrate how recommendation systems can treat each user as a distinct learning task, creating personalized model parameters that generalize across different applications and contexts. It is worth mentioning that the multi-armed bandit approaches in recommendation systems \citep{li2010contextual} naturally embody the exploration-exploitation-creation tradeoffs by continuously balancing known user preferences with the discovery of new content types and recommendation strategies. Our create-to-reuse framework directly parallels these systems' core functionality: recommendation systems routinely invest computational resources in creating personalized embeddings, meta-learned initialization parameters, and graph neural network representations that serve as reusable templates for rapid adaptation to new users, items, and interaction modalities, while continuously expanding their action spaces through dynamic catalog growth and emerging user behavior patterns.

\paragraph{Inventory Management} Inventory management and supply chain systems represent a mature operations research domain where organizations continuously face fundamental tradeoffs between optimizing existing supply chain capabilities and investing in new suppliers, products, or distribution channels. \citet{bellman1958studies} established the mathematical foundations of inventory theory, while dynamic capacity expansion models \citep{mieghem2002newsvendor} demonstrate how firms balance existing capacity utilization with flexible resource investments that create new operational capabilities. The problem of inventory management often coexists with revenue management \citep{chen2019coordinating}, resource allocation \citep{xu2025joint}, and adversarial online learning \citep{xu2025dynamic} that occurs frequently in modern supply chains. The create-to-reuse framework aligns naturally with supply chain decision-making: organizations invest upfront in new suppliers, products, or distribution capabilities that become reusable assets for future deployment across different demand scenarios.



\newpage
\section{Proof Details}
\label{appendix:sec_proof_details}
Here we extend the proof sketch of \Cref{theorem:regret} provided in \Cref{sec:regret_analysis}. According to the roadmap depicted, to validate \Cref{theorem:regret}, we only need to prove the following \Cref{lemma:opt_o_upper_bound,lemma:competitive_ratio_upper_bound,lemma:linear_regression}. We first propose the lemma that bounds $OPT_o$.

\begin{lemma}[$OPT_o$ upper bound]
    \label{lemma:opt_o_upper_bound}
    We have $OPT_o=O(T^{\frac{d}{d+2}}d^{\frac{d}{d+2}})$.
\end{lemma}
\begin{proof}[Proof sketch]
    We propose a context set (library) $\tilde{S}$ such that $c\cdot|\tilde{S}| + \sum_{t=1}^T\min_{f\in\tilde{S}}d(x_t, f)=O(T^{\frac{d}{d+2}}d^{\frac{d}{d+2}})$. Specifically, we let $\tilde{S}:=\{[N_1, N_2,\ldots, N_d]^\top|N_i\in[\frac1{\Delta}], i=1,2,\ldots, d\}$ as a $\Delta$-covering set over the context space of $[0,1]^d$. On the one hand, the cumulative mismatch loss due to discretization of the context space is $O(T\cdot\Delta^2d)$. On the other hand, the total cost of adding new contexts to the set is $O((\frac1{\Delta})^d)$. Let $\Delta=T^{-\frac1{d+2}} d^{-\frac1{d+2}}$ and the total loss is $O(T^{\frac{d}{d+2}}d^{\frac{d}{d+2}})$. Please kindly find a detailed proof in \Cref{appendix:proof_of_lemma_opt_o_upper_bound}.
\end{proof}

Before getting into the main lemma that upper bounds $ALG$, we present another lemma showing the concentration of $\bar{d}_t(x_t, f)$ within $\Delta_t(x_t, f)$.
\begin{lemma}[$\Delta_t$ as estimation error]
    \label{lemma:Delta_t_and_basic_inequalities}
    The estimation error of $|\bar{d}_t(x_t,f)-d(x_t, f)|$ is upper bounded by $\Delta_t(x_t, f)$ with high probability. As a consequence, we have $d(x_t, f)-2\Delta_t(x_t, f)\leq\check{d}_t(x_t, f)\leq d(x_t, f)\leq\hat{d}_t(x_t, f)\leq d(x_t, f) + 2\Delta_t(x_t, f)$.
\end{lemma}
The proof of \Cref{lemma:Delta_t_and_basic_inequalities} is deferred to \Cref{appendix:proof_of_lemma_Delta_t}. 
In the following, we state the lemma that upper bounds the algorithmic loss by a constant competitive ratio over $OPT_o$ adding estimation errors. According to

\begin{lemma}[Constant competitive ratio]
    \label{lemma:competitive_ratio_upper_bound}
    We have $ALG\leq 60 OPT_o + 54 \sum_{t=1}^T\Delta_t(x_t, f_t).$
\end{lemma}

\input{proof_competitive_ratio}
Now we propose the lemma where we upper bound the cumulative estimation error.
\begin{lemma}[Linear regression excess risk]
    \label{lemma:linear_regression}
    The cumulative absolute error of online linear regression with least-square estimator satisfies $\sum_{t=1}^T\Delta_t=O(\sqrt{d^2T \log T})$.
\end{lemma}
We defer the proof of \Cref{lemma:linear_regression} to \Cref{appendix:proof_of_lemma_linear_regression} as a standard result from linear regression. According to what we stated earlier, this completes the proof of \Cref{theorem:regret}.

\bigskip

In the following subsections, we present the proof details of lemmas proposed above.

\subsection{Proof of \texorpdfstring{\Cref{lemma:opt_o_upper_bound}}{Lemma B.1}}
\label{appendix:proof_of_lemma_opt_o_upper_bound}
\begin{proof}
Let $\tilde{S}=\{[N_1, N_2,\ldots, N_d]^\top|N_i\in[\frac1{\Delta}], i=1,2,\ldots, d\}$. On the one hand, for any context $x=[x_1, x_2, \ldots, x_d]^\top\in\R^d, \|x\|_2\leq 1$, consider $f_x:=[\lfloor\frac{x_1}{\Delta}\rfloor\cdot\Delta, \lfloor\frac{x_2}{\Delta}\rfloor\cdot\Delta, \ldots, \lfloor\frac{x_d}{\Delta}\rfloor\cdot\Delta]$. Due to the definition of $\tilde{S}$, we know that $f_x\in\tilde{S}$. Also we have $d(x, f_x) = \|x-f_x\|_{W}^2<\|[\Delta, \Delta, \ldots, \Delta]^\top\|_{W}^2\leq \lambda_{\max}(W)\cdot\Delta^2d$. On the other hand, we have $|\tilde{S}| = (\frac1\Delta)^d$. Denote $S^*$ as the solution to $OPT_o$ (as defined in \Cref{eq:def_s_*}), we have
\begin{equation}
    \label{eq:opt_o_upper_bound_tilde_s}
    \begin{aligned}
        OPT_o=&c\cdot|S^*|+\sum_{t=1}^T\min_{f\in S^*} d(x_t, f)\\
        \leq&c\cdot |\tilde{S}| + \sum_{t=1}^T\min_{f\in \tilde{S}} d(x_t, f)\\
        \leq&c \cdot (\frac1\Delta)^d + \sum_{t=1}^Td(x_t, f_{x_t})\\
        \leq & c(\frac1\Delta)^d + T\cdot\lambda_{\max}(W)\cdot\Delta^2d\\
        =&O(\frac1{\Delta^d} + T\Delta^2 d),
    \end{aligned}
\end{equation}
and we let $\Delta=T^{-\frac1{d+2}}d^{-\frac1{d+2}}$ to make the RHS = $O(T^{\frac {d}{d+2}}d^{\frac{d}{d+2}})$. This proves the lemma.
\end{proof}
\subsection{Proof of \texorpdfstring{\Cref{lemma:Delta_t_and_basic_inequalities}}{Lemma B.2}}
\label{appendix:proof_of_lemma_Delta_t}
\begin{proof}
Here we prove a more general result on ridge regression:
\begin{lemma}
    \label{lemma:ridge_regression_confidence_bound}
    Let $x_1, x_2, \ldots, x_n\in\R^d$ are $d$-dimension vectors, and $y_i:= x_i^\top\theta^* + N_t$, where $\theta^*\in\R^d$ is a fixed unknown vector such that $\|\theta^*\|_2\leq 1$ ,and $N_t$ is a \emph{martingale difference sequence} subject to $\sigma$-subGaussian distributions. Denote $X=[x_1, x_2, \ldots, x_n]^\top \in\R^{n\times d}$ and $Y=[y_1, y_2, \ldots, y_n]^\top\in\R^n$. Let the ridge regression estimator $$\hat\theta:=(X^\top X + I_d)^{-1} X^\top Y$$ where $I_d$ is the $d\times d$ identity matrix. Then with probability $\Pr\geq 1-\delta$, we have
    \begin{equation}
        \label{eq:lemma_ridge_regression}
        \begin{aligned}
            |x^\top(\theta^*-\hat\theta)|\leq O\left((1+\sqrt{\log(\frac2\delta)})\sqrt{x^\top(X^\top X + I_d)^{-1}x}\right)
        \end{aligned}
    \end{equation}
    holds for any $\delta> 0$ and $x\in\R^d$.
\end{lemma}
\begin{proof}[Proof of \Cref{lemma:ridge_regression_confidence_bound}]
    Denote $N:=[N_1, N_2, \ldots, N_n]^\top\in\R^n$ as the vector of noises in the labels. Then we have
    \begin{equation}
        \label{eq:rewrite_theta_hat}
        \begin{aligned}
            \hat{\theta}=(X^\top X + I_d)^{-1}X^\top X\theta^* +(X^\top X + I_d)^{-1}X^\top N.
        \end{aligned}
    \end{equation}
    Therefore, the difference between $\theta^*$ and $\hat\theta$ can be characterized as
    \begin{equation}
        \label{eq:theta_star_hat_theta_difference_1}
        \begin{aligned}
            \theta^* - \hat\theta&=\theta^* - (X^\top X + I_d)^{-1}X^\top X\theta^* - (X^\top X + I_d)^{-1}X^\top N\\
            &=(X^\top X + I_d)^{-1}(X^\top X + I_d) \theta^* - (X^\top X + I_d)^{-1}X^\top X\theta^* - (X^\top X + I_d)^{-1}X^\top N\\
            &=(X^\top X + I_d)^{-1}\theta^* - (X^\top X + I_d)^{-1}X^\top N\\
            &=(X^\top X + I_d)^{-1}(\theta^*-X^\top N).
        \end{aligned}
    \end{equation}
    As a result, we have
    \begin{equation}
        \begin{aligned}
        \label{eq:estimation_error_decomposition}
        |x^\top(\theta^* - \hat\theta)| &= |x^\top(X^\top X + I_d)^{-1}(\theta^*-X^\top N)|\\
        &\leq|x^\top (X^\top X + I_d)^{-1}\theta^*| + |x^\top(X^\top X + I_d)^{-1}X^\top N|.
        \end{aligned}
    \end{equation}
    For the simplicity of notation, denote $A:= (X^\top X + I_d)^-1$, then we have $|x^\top(\theta^* - \hat\theta)|\leq \|x^\top A\theta^*\|_2 + \|x^\top AX^\top N\|$. On the one hand, for the first term we have
    \begin{equation}
        \label{eq:estimation_error_first_term}
        \begin{aligned}
            |x^\top A\theta^*|&\leq \|A^\top x\|_2\cdot\|\theta^*\|_2\\
            &\leq \sqrt{x^\top AA^\top x} \cdot 1\\
            &\leq\sqrt{x^\top A x}.
        \end{aligned}
    \end{equation}
    The second line is because $\|\theta^*\|\leq 1$, and the last inequality is because $A=A^\top = (X^\top X + I_d)^{-1}\prec I_d$.


    On the other hand, for the second term, recall that we set $A := (X^\top X + I_d)^{-1} \quad\text{and}\quad \theta^* - \hat{\theta} = A \left(\theta^* - X^\top N\right).$ We consider the random variable $x^\top A X^\top N = \sum_{t=1}^n \alpha_t N_t,$ where the deterministic coefficients $\alpha_t:=\left(x^\top A X^\top\right)_t, \quad t=1,\ldots,n.$

    Notice that $\{N_t\}$ is a martingale difference sequence with subGaussian tails. According to \citet[Proposition 7]{jin2019short}, which is a subGaussian version of Azuma–Hoeffding's Inequality, let $d=1$ and there exists a constant $C_J$ such that
    \begin{equation}
        \label{eq:martingale_subgaussian_tail_shamir}
        \begin{aligned}
            \Bigl|\sum_{t=1}^n \alpha_t N_t\Bigr|\leq C_J\cdot\sqrt{\sum_{t=1}^n\alpha_t^2\log\frac2\delta}.
        \end{aligned}
    \end{equation}

    with probability $\Pr\geq 1- \delta$. Here $\|\alpha\|_2^2 = \sum_{t=1}^n \alpha_t^2=x^\top A\,X^\top X\,A\,x \leq x^\top A\,x$ because $X^\top X \preceq X^\top X + I_d$.  
    
    %
    Therefore, with probability at least \(1-\delta\),
    \begin{equation}
    \label{eq:estimation_error_second_term_part_2}
    \bigl|x^\top A X^\top N\bigr|\leq C_J\sqrt{\,x^\top A\,x\,\log \left(\frac{2}{\delta}\right)}.
    \end{equation}

    Returning to $|x^\top(\theta^* - \hat\theta)|\leq|x^\top A\,\theta^*|+|x^\top A\,X^\top N|,$ we already established (using $\|\theta^*\|_2 \leq1$) that $|x^\top A\,\theta^*|\leq\sqrt{x^\top A\,x}.$
Combining this with \Cref{eq:estimation_error_second_term_part_2} as the martingale tail bound, we get
\begin{equation}
    \label{eq:estimation_error_therefore}
    \begin{aligned}
        |x^\top(\theta^*-\hat\theta)|\leq& \sqrt{x^\top A\,x}+C_J\cdot\sqrt{\,x^\top A\,x\,\log \left(\frac{2}{\delta}\right)}\\
        =&\left(1+C_J\cdot\sqrt{\,\log \left(\frac{2}{\delta}\right)}\right)\,\sqrt{x^\top (X^\top X + I_d)^{-1}\,x}.
    \end{aligned}
\end{equation}
This ends the proof of \Cref{lemma:ridge_regression_confidence_bound}.

\end{proof}

Now let us go back to the proof of \Cref{lemma:Delta_t_and_basic_inequalities}. We apply this lemma for $6T$ times: in the proof of \Cref{lemma:good_point_total_loss} as Case I(a), I(b), I(c) (or Lemma 5.6) and II, and in the proof of \Cref{lemma:bad_point_individual_loss} as Case I and Case II, in each of which we adopt this concentration bound for each existing context $f\in S_t$, which is at most $T$. Therefore, we let $\delta\leftarrow \frac{1}{6T^2}\delta$ and let $\lambda = 1$, $\alpha = (1 + C_J\cdot\sqrt{\log\frac{12T^2}{\delta}})\cdot\|W\|_F$. According to \Cref{lemma:ridge_regression_confidence_bound}, we prove that $\bar{d}(x_t, f) - \Delta_t(x_t, f)\leq d(x_t, f)\leq \bar{d}(x_t,f) + \Delta_t(x_t, f)$ holds for any $f\in S_t$ and $\forall t=1,2,\ldots, T$, with probability $\Pr\geq 1-\delta$. Therefore, we have proved \Cref{lemma:Delta_t_and_basic_inequalities}.
\end{proof}
\subsection{Proof of \texorpdfstring{\Cref{lemma:constant_loss_bound_before_a_new_arm_emerges}}{Lemma B.5}}
\label{appendix:proof_of_lemma_constant_loss_before_new_arm}
\begin{proof}
            Notice that at each time $t_k$, with probability $\Pr=\frac{\hat d_{t_k}(x_{t_k}, f_{t_k})}c$ we terminate this stochastic process, and with the rest $\Pr=1-\frac{\hat d_{t_k}(x_{t_k}, f_{t_k})}c$ we add $d_{t_k}(x_{t_k}, f_{t_k})$ to our cumulative expected loss. Since $\hat d_{t_k}(x_{t_k}, f_{t_k})\geq d_{t_k}(x_{t_k}, f_{t_k}), \forall k\in[n]$, we may instead prove a generalized version of this lemma.
            \begin{lemma}
                \label{lemma:generalized_cumulative_upper_bound}
                Consider an infinite sequence $\{p_1, p_2, \ldots, p_k, \ldots\}$ where $p_k\in [0,1]$. The initial sum $S = 0$. At each time $k$, with probability $p_k$ we \emph{stop} this stochastic process, otherwise we add $p_k$ to the sum $S$. We show that $\E[S]\leq 1$.
            \end{lemma}
            \Cref{lemma:generalized_cumulative_upper_bound} is a generalization of \Cref{lemma:constant_loss_bound_before_a_new_arm_emerges} since we add $d_{t_k}(x_{t_k}, f_{t_k})\leq \hat d_{t_k}(x_{t_k}, f_{t_k})$ at each time $k$ in the latter setting.

            Denote a random variable $I_k$ as follows: $I_k = 1$ if the stochastic process has not stopped by the end of time $k$, and $I_k = 0$ otherwise. In the case when $I_k=1$, we add $p_k$ to the sum $S$. Therefore, we have $$\E[S] = \sum_{k=1}^{\infty} p_k I_k$$. Also, we know that the probability that $I_k = 1$ is $\Pr[I_k = 1]= \prod_{i=1}^{k}(1-p_i).$ As a result, we have
            \begin{equation}
                \label{eq:infinite_sum}
                \begin{aligned}
                    \E[S] &= \E[\sum_{k=1}^{\infty}p_k\cdot I_k]  \\
                    & = \sum_{k=1}^{\infty}p_k\prod_{i=1}^k (1-p_i).
                \end{aligned}
            \end{equation}
            In the following, we show that $\sum_{k=1}^{\infty}p_k\prod_{i=1}^k (1-p_i)\leq 1$. We first consider $p_k\in(0,1)$. Denote $Q_0 := 1$ and $Q_k:=\Pr[I_k=1]=\prod_{i=1}^k (1-p_i)$, and we know $Q_k=(1-p_k) Q_{k-1}\leq Q_{k-1}$. Also, we have $p_kQ_{k-1}=(1-(1-p_k))Q_{k-1}=Q_{k-1}-Q_k$.
            
            For the rigorousness of the proof, we first show that $\sum_{k=1}^{\infty}p_kQ_k$ is finite. Denote
            \begin{equation}
                \label{eq:t_n_sum_p_k_Q_k}
                T_n:= \sum_{k=1}^n p_kQ_k,
            \end{equation}
            and we have
            \begin{equation}
                \label{eq:T_n_upper_bound}
                \begin{aligned}
                    T_n\leq& \sum_{k=1}^n p_k Q_{k-1}\\
                    =& \sum_{k=1}^n Q_{k-1}-Q_k\\
                    =& Q_0 - Q_n\\
                    <& Q_0 = 1
                \end{aligned}
            \end{equation}
            As $T_n<1$ and $T_{n+1}\geq T_n, \forall n\geq 1$, we have
            \begin{equation}
                \label{eq:T_n_lim_exists_leq_1}
                \lim_{n\rightarrow\infty} T_n \leq 1
            \end{equation}
            according to the Monotone Convergence Theorem. 
            Then we slightly generalize the results above from $p_k\in(0,1)$ to $p_k\in[0,1]$, i.e. incorporating $0$ and $1$. In fact, if $p_k=0$, then we may skip this $p_kQ_k$ term. Otherwise if $p_k=1$, consider the first $m$ s.t. $p_m=1$, and then we still have $E[S]=\sum_{k=1}^{m-1} p_k I_k = T_{m-1} < 1$ and $I_m=I_{M}=0$ for any $M\geq m, M\in\mathbb{Z}^+$. 
            
            Therefore, we have
            \begin{equation}
                \label{eq:infinite_sum_inequality_final}
                \begin{aligned}
                    \E[S]=&\sum_{k=1}^{\infty} p_kQ_k\\
                    \leq&\sum_{k=1}^{\infty} p_kQ_{k-1}\\
                    =&\sum_{k=1}^{\infty} Q_{k-1} - Q_k\\
                    =&Q_0 - \lim_{k\rightarrow\infty} Q_k\\
                    \leq& 1.
                \end{aligned}
            \end{equation}
            This ends the proof of \Cref{lemma:generalized_cumulative_upper_bound} and therefore proves \Cref{lemma:constant_loss_bound_before_a_new_arm_emerges}.
        \end{proof}
\subsection{Proof of \texorpdfstring{\Cref{lemma:bad_point_individual_loss}}{Lemma B.6}}
\label{appendix:proof_of_lemma_bad_point_individual_loss}
\begin{proof}
    Consider the moment when a $x_t\in C_i^b$ arrives, and denote $s$ as the most recent moment ($s<t$) such that $x_s\in C_i^g$. According to the uniform permutation assumption from $Z$ to $\{x_t\}_{t=1}^T$, this $x_s$ can be any $z\in C_i^g$ with equal probability as $\Pr=\frac1{|C_i^g|}=\frac2{|C_i^*|}$. 
    In the following, we analyze the expected loss $\E[l_t]$ by two cases:
        \begin{enumerate}[label=(\Roman*)]
            \item If $x_s\in C_i^g$ does exist before $x_t$ occurs. Denote $f_t^*:=\argmin_{f\in S_t} d(c_i^*, f)$ as the closest context to $c_i^*$ existed by the time $t$. Then we have:
            \begin{equation}
                \label{eq:bad_point_case_1_l_t_upper_bound}
                \begin{aligned}
                    \E[l_t|\{x_t\}_{t=1}^T]= & c\cdot\frac{\hat d_t(x_t, f_t)}{c} + d(x_t, f_t)\cdot(1-\frac{\hat d_t(x_t, f_t)}{c})\\
                    \leq &\hat d_t(x_t, f_t) + d(x_t, f_t)\\
                    \leq& \check d_t(x_t, f_t) + 2\Delta_t + \check d_t(x_t, f_t) + 2\Delta_t\\
                    \leq &2\check d_t(x_t, f_t^*) + 4\Delta_t\\
                    \leq &2d(x_t, f_t^*) + 4\Delta_t\\
                    \leq &4d(x_t, c_i^*) + 4d(c_i^*, f_s^*) + 4\Delta_t.
                \end{aligned}
            \end{equation}
            Also, denote $\hat f_s:=\argmin_{f\in S_s} d(x_s, f)$ as the best existing context that can be matched to $x_s$ by the time $s$. Then we know that
            \begin{equation}
                \label{eq:bad_point_case_1_l_s_lower_bound}
                \begin{aligned}
                    \E[l_s|\{x_t\}_{t=1}^T]&\geq d(x_s, \hat f_s)\\
                    &\geq \frac12 d(c_i^*, \hat f_s) - d(x_s, c_i^*)\\
                    &\geq \frac12 d(c_i^*, f_s^*) - d(x_s, c_i^*).
                \end{aligned}
            \end{equation}
            Combining \Cref{eq:bad_point_case_1_l_t_upper_bound} with \Cref{eq:bad_point_case_1_l_s_lower_bound}, we have
            \begin{equation}
                \label{eq:bad_point_case_1_l_t_final_upper_bound}
                \begin{aligned}
                    \E[l_t|\{x_t\}_{t=1}^T]&\leq 4d(x_t, c_i^*) + 4\Delta_t + 4\cdot2(\E[l_s|\{x_t\}_{t=1}^T] + d(x_s, c_i^*))\\
                    &\leq 4d_t^* + 2\Delta_t + 8\cdot \frac2{|C_i^*|}\cdot(\sum_{s:x_s\in C_i^g}\E[l_s|\{x_t\}_{t=1}^T] + d_s^*).
                \end{aligned}
            \end{equation}
            Again, the last line of \Cref{eq:bad_point_case_1_l_t_final_upper_bound} comes from the i.i.d. assumption of $x_s$.
            \item If $x_s\in C_i^g$ does not exist before $x_t$ occurs, i.e. $x_r\in C_i^b, \forall r\leq t-1$. According to the uniform permutation from $Z$ to $\{x_t\}_{t=1}^T$, this event happens with probability $\frac2{|C_i^*|}$. In this case, if $\hat d_t(x_t, f_t)\geq c$, then we suffer a cost $c$ at time $t$; otherwise $\hat d_t(x_t, f_t) < c$, and we either generate a new action (with cost $c$) or suffer an expected loss at $d(x_t, f_t)\leq \hat d_t(x_t, f_t) < c$. In a nutshell, the expected loss does not exceed $c$.
        \end{enumerate}
        Combining with Case I and Case II above, we immediately get \Cref{eq:bad_point_individual_loss}.
\end{proof}
\subsection{Proof of \texorpdfstring{\Cref{lemma:linear_regression}}{Lemma B.8}}
\label{appendix:proof_of_lemma_linear_regression}
\begin{proof}
Denote $\Delta_t:=\Delta_t(x_t, f_t)$ for simplicity. In the following, we first reduce the summation of estimation error $\Delta_t$ to the regret of a $K(\leq T)$-arm linear bandit problem, up to constant factors. In fact, according to \Cref{lemma:Delta_t_and_basic_inequalities}, we know that $d(x_t, f)\leq\check{d}_t(x_t, f)\leq d(x_t, f)$. Since we select $f_t=\argmin_{f\in S_t}\check{d}_t(x_t, f)$, we have
\begin{equation}
\label{eq:optimistic_gap_upper_bound}
\begin{aligned}
d(x_t, f_t^*)-2\Delta_t\leq& d(x_t, f_t)-2\Delta_t\\
\leq&\check{d}_t(x_t, f_t)\\
\leq&\check{d}_t(x_t, f_t^*)\\
\leq& d(x_t, f_t^*),
\end{aligned}
\end{equation}

where $f_t^*:=\argmin_{f\in S_t}d(x_t, f)$ as the best existing context for $x$ in the current context library at time $t$. Therefore, the performance gap between $f_t$ and $f_t^*$ can be bounded as $d(x_t, f_t)-d(x_t, f_t^*)\leq 2\Delta_t$. On the other hand, since $d(x_t, f_t) = <w, \phi(x_t, f_t)>$ is a linear loss function, we consider $\phi(x_t, f)$ as the ``context''\footnote{Here we denote this covariate as the \emph{context} as it serves as an environmental description in the contextual bandits. We denote $f\in S_t$, which was denoted as a context in the library, an \emph{arm} of this contextual bandits.} 
of each arm $f\in S_t$, and then we form a linear contextual bandit problem setting. Recall that $\Delta_t=\alpha\cdot\sqrt{\phi(x_t, f_t)^\top\Sigma_{t-1}^{-1}\phi(x_t, f_t)}$. According to \citet{chu2011contextual} Lemma 3 (which originates from \citet{auer2002using} Lemma 13), we have
\begin{equation}
    \label{eq:chu11a_linear_ucb}
    \begin{aligned}
        \sum_{t=1}^T\Delta_t & =\sum_{t=1}^T \alpha\cdot\sqrt{\phi(x_t, f_t)^\top\Sigma_{t-1}^{-1}\phi(x_t, f_t)}\\
        &\leq \alpha\cdot5\sqrt{(d^2)|\Psi_{T+1}|\log|\Psi_{T+1}|}\\
        &\leq5\alpha\sqrt{d^2T\log T}.
    \end{aligned}
\end{equation}
Here the second line is because the dimension of contexts are $d^2$ as $\phi(x_t, f)=Vec((x_t-f)(x_t-f)^\top)\in\R^{d^2}$, and the third line comes from the original definition of $\Psi_t$ as a subset of $[t-1]$.
\end{proof}
\subsection{Proof of Lower Bound (\texorpdfstring{\Cref{theorem:lower_bound}}{Theorem 5.5})}
\label{appendix:proof_of_lower_bound}
\begin{proof}

     In order to prove the lower bound, we show the following facts
\begin{enumerate}
    \item $OPT_o=\Omega(T^{\frac{d}{d+2}})$ according to the $K$-nearest-neighbors(K-NN) lower bound.
    \item Any online facility location algorithm suffers at least $(2-o(1))$-competitive-ratio, i.e. $ALG\geq (2-o(1)) OPT_h$.
\end{enumerate}
In the following, we present two lemmas corresponding to the facts above.

\begin{lemma}[$OPT_o$ lower bound]
    \label{lemma:opt_o_lower_bound}
    We have $OPT_h\geq OPT_o \geq \Omega(T^{\frac{d}{d+2}})$.
\end{lemma}
\begin{proof}
    Denote 
    \begin{equation}
        \label{eq:opt_o_k}
        \begin{aligned}
        OPT_o(K)&:=\min_{S: |S|=K} c\cdot |S| + \sum_{t=1}^T\min_{f\in S} d(x_t, f)\\
        &= K + T\cdot\min_{S: |S|=K}\frac1T\sum_{t=1}^T\min_{f\in S} d(x_t, f).
        \end{aligned}
    \end{equation}
    This equals $T$ times $K$-nearest-neighbors (K-NN) loss plus $K$. According to \citet{zador1964development} (i.e. Zador's Theorem in coding theory), the mean squared distance to the nearest codebook center in $\R^d$ space in $L_r$-norm is lower bounded by $\Omega(K^{-\frac{r}{d}})$. This is directly applicable to K-NN which effectively partitions points by their nearest neighbors. Hence, the quantization lower bound established by Zador’s Theorem translates into a lower bound on K-NN’s average squared loss. 
    Therefore, we let $r=2$ to fit in our setting, and then have
    \begin{equation}
        \label{eq:opt_o_lower_bound_from_k_means}
        \begin{aligned}
            OPT_o=&\min_{K\in[T]}OPT_o(K)
            =\Omega(c\cdot K + T\cdot K^{-\frac2d})
            =\Omega(T^{\frac{d}{d+2}}),
        \end{aligned}
    \end{equation}
    where the last line is an application of Hölder's Inequality that $K + T\cdot K^{-\frac2d}\geq K^{\frac{\frac2d}{1+\frac2d}}(T\cdot K^{-\frac2d})^{\frac1{1+\frac2d}}=T^{\frac{d}{d+2}}$, and the equality holds at $K=T^{\frac{d}{d+2}}$.
\end{proof}
\begin{lemma}[Theorem 5.1 in \citet{kaplan2023almost}]
    \label{lemma:competitive_ratio_facility_location_lower_bound}
    Let $\cA$ be an algorithm for online facility location in the i.i.d. model, then, the competitive ratio of $\cA$ is at least $2-o(1)$.
\end{lemma}
Combining \Cref{lemma:opt_o_lower_bound} and \Cref{lemma:competitive_ratio_facility_location_lower_bound}, we know that $REG=ALG-OPT_h \geq (2-o(1)-1)OPT_h\geq 0.5 OPT_o = \Omega(T^{\frac{d}{d+2}})$. This proves \Cref{theorem:lower_bound}
\end{proof}

%% file: proof_competitive_ratio.tex
\begin{proof}
    Before starting the proof, we emphasize that all operations we make in this proof are made in the \emph{context} space. As we frequently mention in this paper, the actions are only accessible through the oracle $\cA(x)$ for some context $x$. Therefore, the context library $S_t$ is sometimes referred as a ``set'' without causing misunderstandings.

    First of all, we note that the following two $\{x_t\}_{t=1}^T$ series have identical joint distributions:
    \begin{enumerate}[label=(\alph*)]
        \item Sample a sequence of $x_1, x_2, \ldots, x_T$ independently from an identical distribution $\mathbb{D}_X$. (iid)
        \item Sample a set of $Z:=\{z_1, z_2, \ldots, z_T\}$ independently from the identical distribution $\mathbb{D}_X$, and then sample $\{x_t\}_{t=1}^T$ as a uniformly random permutation of $Z$, i.e. $\{x_t\}_{t=1}^T \sim U(\sigma(Z))$. Here $\sigma(Z)$ denotes the permutation set of $Z$. (iid + permutation)
    \end{enumerate}
    Given this property, we assume that $\exists Z=\{z_1, z_2, \ldots, z_T\}, z_t \overset{\text{i.i.d.}}{\sim} \mathbb{D}_X, \{x_t\}_{t=1}^T\sim U(\sigma(Z))$.
    In the following, we will keep using the notations of $\{x_t\}_{t=1}^T$ and $Z$ accordingly.
    
    Consider the optimal offline solution $S^*$ such that
    \begin{equation}
        \label{eq:def_s_*}
        \begin{aligned}
            OPT_o&=c\cdot|S^*(x_1, x_2, \ldots, x_T)| + \sum_{t=1}^T \min_{f\in S^*} d(x_t, f)\\
            &=c\cdot|S^*(z_1, z_2, \ldots, z_T)| + \sum_{t=1}^T \min_{f\in S^*} d(x_t, f).
        \end{aligned}
    \end{equation}
    
    Here we denote $S^*(x_1, x_2, \ldots, x_T)$ and $S^*(z_1, z_2, \ldots, z_T)$ differently to show that the offline solution is not dependent on the permutation, with slight abuse of notation. Denote $S^*=:\{c_1^*, c_2^*, \ldots, c_K^*\}$. For each $c_i^*, i=1,2,\ldots, K$, denote a subset of $\{x_t\}$ as $C_i^*$ such that $\min_{f\in S^*} d(x_t, f) = d(x_t, c_i^*), \forall x_t\in C_i^*$. In other words, $C_i^*$ consists of all $x_t$'s that are assigned to $c_i^*$ in the optimal solution $S^*$. Denote $A_i^*:=\sum_{t: x_t\in C_i^*} d(x_t, c_i^*)$ as the total optimal cost associated with $c_i^*$, and $a_i^*:=\frac{A_i^*}{|C_i^*|}$ as the average cost in $C_i^*$. 
    
    Now, we define $C_i^g$ and $C_i^b$ as separated \textsc{good} and \textsc{bad} subsets of $C_i^*$, respectively, such that
    \begin{equation}
        \label{eq:def_good_bad_subsets}
        \begin{aligned}
            C_i^g\subset& C_i^*,\ C_i^b\subset C_i^*,\ |C_i^g|=|C_i^b|=\frac{|C_i^*|}2\\
            d(x_g,c_i^*)\leq& d(x_b, c_i^*), \forall x_g\in C_i^g, x_b\in C_i^b.
        \end{aligned}
    \end{equation}
    
    In other words, $C_i^g$ and $C_i^b$ represent the nearest half and the farthest half of $x_t$'s in the set $C_i^*$, in terms of distance to $c_i^*$. Note that the sets $C_i^g$ and $C_i^b$ are determined by $Z$ and not relevant to the permutation. Therefore, once $Z$ is realized, the random sequence $\{x_t\}_{t=1}^T$ does not affect $C_i^g$ and $C_i^b$.
    
    Given these notations, we present and prove the following two lemmas: a \Cref{lemma:good_point_total_loss} bounding the \emph{total} loss of \textsc{good} $x_t$'s, and a \Cref{lemma:bad_point_individual_loss} bounding the \emph{individual} loss of each \textsc{bad} $x_t$'s. 
    \begin{lemma}
        \label{lemma:good_point_total_loss}
        The total loss caused by all $x_t\in C_i^g$ is upper bounded as
        \begin{equation}
            \label{eq:good_point_total_loss_lemma}
            \sum_{t:x_t\in C_i^g} \E[l_t|\{x_t\}_{t=1}^T] \leq 3c + 4A_i^* + 4\sum_{x_t\in C_i^g} d(x_t, c_i^*) + 6\sum_{t=1}^T\Delta_t(x_t, f_t).
        \end{equation}
    \end{lemma}
    \begin{proof}[Proof of \Cref{lemma:good_point_total_loss}]
        Denote the context set (library) sequence as $\{S_t\}_{t=1}^T$. Also, denote $\Delta_t:=\Delta_t(x_t, f_t)$ and $d_t^*:=d(x_t, c_i^*)$ for simplicity. 
        In fact, any $x_t\in C_i^g$ falls in one of the following two cases:
        
        \begin{enumerate}[label=(\Roman*)]
            \item When $\exists\ e_i\in S_T$ such that $d(e_i, c_i^*)\leq 2a_i^*$, we further categorize $x_t$ into three sub-cases:
            
            \textbf{I.(a).} At time $t$, we select context $e_i$ and deploy $a_t = \cA(e_i)$ (i.e., $x_t$ is matched to context $e_i$). We have
                \begin{equation}
                    \label{eq:case_1_a_upper_bound_d}
                    \begin{aligned}
                        d(x_t, e_i)\leq2(d(x_t, c_i^*) + d(c_i^*, e_i))\leq 2(d_t^* + 2a_i^*).
                    \end{aligned}
                \end{equation}
                The first inequality is due to
                \begin{equation}
                    \label{eq:d_x_f_triangular}
                    d(a,b) + d(b,c) \geq \frac12d(a,c), \forall a, b, c\in\R^{d^2}.
                \end{equation}
                as a quadratic form. Hence
                \begin{equation}
                    \label{eq:case_1_a_upper_bound_l_t}
                    \begin{aligned}
                        \E[l_t|\{x_t\}_{t=1}^T]\leq 2d(x_t, e_i) + 2\Delta(x_t, f_t)\leq 4(d_t^* + 2a_i^*) + 2\Delta_t.
                    \end{aligned}
                \end{equation}
                \textbf{I.(b).} At time $t$, $e_i\in S_t$ but $a_t\neq \cA(e_i)$, i.e. $x_t$ is matched to some other context $f_t$ even with the existence of $e_i$. Now we have
                \begin{equation}
                    \label{eq:case_1_b_upper_bound_d_part_1}
                    \begin{aligned}
                        d(x_t, f_t)-2\Delta_t \leq \check{d}_t(x_t, f_t) \leq \check{d}_t(x_t, e_i)\leq d(x_t, e_i).
                    \end{aligned}
                \end{equation}
                The second inequality comes from the arg-minimum definition of $f_t$, and the first and third inequalities is from \Cref{lemma:Delta_t_and_basic_inequalities}. Therefore, we have
                \begin{equation}
                    \label{eq:case_1_b_upper_bound_d_part_2}
                    \begin{aligned}
                        d(x_t, f_t)&\leq d_t(x_t, e_i) + 2\Delta_t
                        \leq 2 (d(x_t, c_i^*) + d(c_i^*, e_i)) + 2\Delta_t
                        \leq 2(d_t^* + 2a_i^*) + 2\Delta_t
                    \end{aligned}
                \end{equation}
                Hence we have
                \begin{equation}
                    \label{eq:case_1_b_upper_bound_l_t}
                    \begin{aligned}
                        \E[l_t|\{x_t\}_{t=1}^T]\leq 2d(x_t, e_i) + 2\Delta_t \leq 4(d_t^* + 2a_i^*) + 6\Delta_t.
                    \end{aligned}
                \end{equation}
                \textbf{I.(c).} $e_i\notin S_t$ at time $t$, i.e. $x_t$ is matched to some $f_t$ before any close-enough context $e_i$ being added. In this case, we propose the following lemma that provides an \emph{overall} loss bound for any group of $\{x_t\}$'s, on which no new actions have been created.
                
                \begin{lemma}[Constant loss bound before a new action being generated]
                    \label{lemma:constant_loss_bound_before_a_new_arm_emerges}
                        Denote $Q:=\{x_{t_i},\ i=1,2,\ldots, n| 1\leq t_1\leq \ldots\leq t_n\leq T\}$ as a subsequence of $\{x_t\}_{t=1}^T$. Also, denote $t_k$ as the \emph{first} time in $Q$ such that a new action is generated, i.e. $a_{t_k}=\cA(x_{t_k})$ and $a_{t_i}\neq \cA(x_{t_i}), i\leq k-1$. We have
                    \begin{equation}
                        \label{eq:loss_bound_before_new_arm_generated}
                        \E[\sum_{i=1}^{k-1}l_{t_i}|\{x_t\}_{t=1}^T]\leq c.
                    \end{equation}
                \end{lemma}
        
                We defer the proof of \Cref{lemma:constant_loss_bound_before_a_new_arm_emerges} to \Cref{appendix:proof_of_lemma_constant_loss_before_new_arm}, where we will prove a generalized claim.
                According to \Cref{lemma:constant_loss_bound_before_a_new_arm_emerges}, the total expected loss for all $x_t$ in this case can be bounded by $c$. 
            \item When $\forall\ e\in S_T$ satisfies $d(e, c_i^*)>2a_i^*$, we know that no new action are generated at time $t$, $\forall t:\  x_t\in C_i^g$. Then we again apply \Cref{lemma:constant_loss_bound_before_a_new_arm_emerges} and upper bound the expected total loss by $c$. 
        \end{enumerate}
        Combining Case I (a,b,c) and Case II, along with a separate cost $c$ of adding $e_i$, we have an upper bound on the expected total loss for all $t:\ x_t\in C_i^g$ as follows:
        \begin{equation}
            \label{eq:good_point_total_loss_result}
            \begin{aligned}
            \E[\sum_{t: x_t\in C_i^g} l_t]&\leq4\sum_{t: x_t\in C_i^g} d_t^* + 8\sum_{t: x_t\in C_i^g} a_i^* + 6\sum_{t: x_t\in C_i^g} \Delta_t + 3c\\
            &= 4\sum_{t: x_t\in C_i^g} d_t^* + 4A_i^* + 6\sum_{t: x_t\in C_i^g} + 3c.
            \end{aligned}
        \end{equation}
        Here the last line comes from $|C_i^g|=\frac{|C_i^*|}2$. This proves \Cref{lemma:good_point_total_loss}.
    \end{proof}
    The previous lemma bounds the \emph{total} loss of \textsc{good} $x_t$'s, while the following lemma will bound the \emph{individual} loss of \textsc{bad} $x_t$'s''.
    \begin{lemma}
        \label{lemma:bad_point_individual_loss}
        For each individual $x_t\in C_i^b$, the expected loss is upper bounded as
        \begin{equation}
            \label{eq:bad_point_individual_loss}
            \E[l_t|Z]\leq 4d(x_t, c_i^*) + 4 \Delta_t(x_t, f_t) + \frac2{|C_i^*|}\cdot(c+8\sum_{s: x_s\in C_i^g}\E[l_s|Z] + 8\sum_{s: x_s\in C_i^g} d(x_s, c_i^*)).
        \end{equation}
    \end{lemma}
    \begin{proof}[Proof sketch of \texorpdfstring{\Cref{lemma:bad_point_individual_loss}}]
        Intuitively, \emph{later}-arrived $x_t$'s should be facing a \emph{better} situation as there are more action candidates. Therefore, for any $x_t\in C_i^b$, if there exists a good point $x_g$ that emerges before the occurrence of $x_t$, we can upper bound $\E[l_t]$ with $\E[l_g]$ adding $d(x_t, c_i^*)$. This is because we can at least match $x_t$ to the existing in-library context that $x_g$ was matched to. Denote $f_g$ as the existing context whose custom action $x_g$ was assigned to. According to the ``triangular inequality'' shown as Equation (13) (up to constant coefficient), we have: $E[l_t]\leq O(d(x_t, f_g))\leq O(d(x_t, c_i^\star) + d(c_i^\star, f_g))\leq O(d(x_t, c_i^\star) + d(c_i^\star, x_g) + d(x_g, f_g)) = O(d(x_t, c_i^\star) + d(x_g, c_i^\star) + E[l_g])$.
        If there is no such a $x_g$ (with very small probability), then we upper bound the expected loss by $c$. For a detailed proof of \Cref{lemma:bad_point_individual_loss}, please kindly refer to \Cref{appendix:proof_of_lemma_bad_point_individual_loss}.
    \end{proof}
    
    Combining \Cref{lemma:good_point_total_loss} and \Cref{lemma:bad_point_individual_loss} above, we have
    \begin{equation}
        \label{eq:competitive_ratio_upper_bound}
        \begin{aligned}
            &\E[\sum_{t:x_t\in C_i^*}l_t]\\
            =&\E[\E[\sum_{t:x_t\in C_i^*}l_t|Z]]\\
            =&\E[\E[\sum_{s: x_s\in C_i^g} l_s|Z] + \E[\sum_{r:x_r\in C_i^b} l_r|Z]]\\
            =&\E[\E[\sum_{t:x_t\in C_i^*}l_t|\{x_t\}_{t=1}^T] + \E[\sum_{r:x_r\in C_i^b} l_r|Z]]\\
            \leq & \E[3c + 4A_i^* + 4\sum_{s: x_s\in C_i^g} d(x_s, c_i^*) + 6\sum_{s: x_s\in C_i^g}\Delta_s(x_s, f_s)]\\
            &+ \E[4\sum_{r: x_r\in C_i^b}d(x_r, c_i^*) + 4 \sum_{r: x_r\in C_i^b} \Delta_r(x_r, f_r)\\
            &+ \frac{|C_i^*|}2\cdot\frac2{|C_i^*|}\cdot(c+8\sum_{s: x_s\in C_i^g}\E[l_s] + 8\sum_{s: x_s\in C_i^g} d(x_s, c_i^*))]\\
            \leq&\E[28c + 40 A_i^* + 40 \sum_{s:x_s\in C_i^g} d(x_s, c_i^*) + 54\sum_{t:x_t\in C_i^*}\Delta_t(x_t, f_t)]\\
            \leq& \E[ 28c + 60 A_i^* + 54\sum_{t:x_t\in C_i^*}\Delta_t(x_t, f_t)].
        \end{aligned}
    \end{equation}
    Here the last inequality is because $\sum_{s:x_s\in C_i^g} d(x_s, c_i^*)\leq\frac{\sum_{s:x_s\in C_i^g} + \sum_{r:x_r\in C_i^b}}2=\frac{A_i^2}2$. On the other hand, the sum of losses in $OPT_o$ that are associated to $c_i^*$ equals $c + A_i^*$. Therefore, we have $ALG\leq 60 OPT_o + 54 \sum_{t=1}^T\Delta_t(x_t, f_t)$. This ends the proof of \Cref{lemma:competitive_ratio_upper_bound}.
\end{proof}
\begin{remark}
    The reason for us to divide $\{x_t\}$'s into \textsc{good} and \textsc{bad} subsets is twofold.
    \begin{enumerate}[label=(\arabic*)]
        \item We can upper-bound the \emph{total} loss of all \textsc{good} points, mainly because we have \Cref{lemma:constant_loss_bound_before_a_new_arm_emerges} such that the Case I(c) and Case II hold. \Cref{lemma:constant_loss_bound_before_a_new_arm_emerges} states that for any group of $\{x_t\}$'s, the expected cost before a new action being created (i.e. before a new context is added to the library) among them is no more than $c$. Therefore, if there does not exist an $e_i$ close enough to $c_i^*$, we know that no new actions have been created among \textsc{good} $\{x_t\}$'s (since any \textsc{good} $x_t$ satisfies $d(x_t, c_i^*)\leq 2a_i^*$ and therefore is a qualified candidate $e_i$ once being added to the existing context library). However, this does not hold for \textsc{bad} points, as they may still trigger new action generations although their contexts are faraway from $c_i^*$.
        \item We can only upper-bound the \emph{individual} loss of each \textsc{bad} $x_t$ due to the reason in (1) above. The individual upper bound for a \textsc{bad} point is applicable for a \textsc{good} point, but this would introduce a linear dependence on $T\cdot c$ in the overall loss instead of a constant ratio.
    \end{enumerate}
    
\end{remark}